\documentclass{article}

\PassOptionsToPackage{numbers}{natbib}
\usepackage[utf8]{inputenc}
\usepackage[T1]{fontenc}
\usepackage{hyperref}
\usepackage{url}
\usepackage{booktabs}
\usepackage{amsfonts}
\usepackage{nicefrac}
\usepackage{microtype}
\usepackage{xcolor}
\usepackage{array}
\usepackage{graphicx}
\usepackage{amsmath, amsthm, amssymb}
\usepackage{graphicx}
\usepackage{subcaption}
\usepackage{wrapfig}
\usepackage{float} % Required for [H]
 \usepackage[preprint]{neurips_2026}

\usepackage[utf8]{inputenc} % allow utf-8 input
\usepackage[T1]{fontenc}    % use 8-bit T1 fonts
\usepackage{hyperref}       % hyperlinks
\usepackage{url}            % simple URL typesetting
\usepackage{booktabs}       % professional-quality tables
\usepackage{amsfonts}       % blackboard math symbols
\usepackage{nicefrac}       % compact symbols for 1/2, etc.
\usepackage{microtype}      % microtypography
\usepackage{xcolor}         % colors
\newtheorem{theorem}{Theorem}

\newtheorem{lemma}{Lemma}
\newtheorem{corollary}{Corollary}
\newtheorem{definition}{Definition}
\newtheorem{remark}{Remark}

\title{The Inverse-Wisdom Law: Architectural Tribalism and the Consensus Paradox in Agentic Swarms}

\vspace{-0.9cm}
\author{%
  Dahlia Shehata \\
  \texttt{dahlia.shehata@uwaterloo.ca} \\
  University of Waterloo \\
  Canada \\
  \And
  Ming Li \\
  \texttt{mli@uwaterloo.ca} \\
  University of Waterloo \\
  Canada 
}

\begin{document}
\vspace{-0.7cm}

\maketitle
\vspace{-1cm}
\begin{abstract}
\vspace{-.2cm}
As AI transitions toward multi-agent systems (MAS) to solve complex workflows, research paradigms operate on the axiomatic assumption that agent collaboration mirrors the ``Wisdom of the Crowd.''
We challenge this assumption by formalizing the \textit{Consensus Paradox}: a phenomenon where agentic swarms prioritize internal architectural agreement over external logical truth.
Through a 36 experiments encompassing $12,804$ trajectories across three state-of-the-art (SOTA) benchmarks (GAIA, Multi-Challenge, and SWE-bench), we prove the \textit{Inverse-Wisdom Law}: in kinship-dominant swarms, adding logical agents increases the stability of erroneous trajectories rather than the probability of truth. The introduction of additional logical audits converges the system toward a \textit{Logic Saturation} where internal entropy hits zero while factual error hits unity.
By evaluating the interaction between the 3 preeminent SOTA models---Gemini 3.1 Pro, Claude Sonnet 4.6, and GPT-5.4---we establish the \textit{Architectural Tribalism Asymmetry} as a mechanistic law of transformer weights. We demonstrate that terminal swarm integrity is strictly gated by the synthesizer’s receptive logic, rather than aggregate agent quality.
% identifying a primary technical blocker for achieving L5-level autonomous governance.
We define the \textit{Tribalism Coefficient} ($\tau$) and the \textit{Sycophantic Weight} ($\sigma$) as the primary mechanistic determinants of swarm failure. Finally, we establish the \textit{Heterogeneity Mandate} as a foundational safety requirement for resilient agentic architectures.

\end{abstract}

\vspace{-0.6cm}
\section{Introduction}
\vspace{-0.25cm}

AI deployment is undergoing a metamorphosis from stateless processors to goal-directed MAS \cite{Kim2025TowardsAS, zhang-etal-2025-swarmagentic}.
``Agent swarms'' are projected to become the operational standard for the digital workforce \cite{NISA202669}. 
However, this shift is accompanied by a profound reliability gap: early-stage agents fail in approximately 70\% of real-world tasks due to context decay and the inability to navigate specialized data silos \cite{xu2025theagentcompany}.
Extensive research \cite{li2024more, gosmar2025hallucination, 10.1016/j.eswa.2024.125723, chan2024chateval, xu-etal-2025-towards} investigated multi-agent architectures 
% as a means to facilitate iterative debate and
to mitigate individual model hallucinations by leveraging the 'Wisdom of the Crowd'.
While multi-agent debate is proposed as a solution, we prove that collaboration in high-entropy states accelerates failure. We challenge the classical Inverted U-shape assumption where collective intelligence peaks at an optimal crowd size \cite{pezzotta2026increasing}.
We define the primary barrier to swarm reliability as the \textit{Consensus Paradox}.
% : the tendency of models to form a ``United Front'' around an initial error, ignoring logically superior corrections if they originate from an architectural stranger.
In such environments, agents exhibit \textit{Sycophantic Consensus}---prioritizing group agreement to minimize the computational entropy of the interaction rather than challenging a peer's micro-hallucination. This behavioral bias leads to \textit{Hallucination Cascades}, where one error poisons the collective memory.
% acts as a ``sleeper instruction'' that 

% While classical theories of collective intelligence suggest an \textbf{Inverted U-shape} relationship where wisdom peaks at an optimal crowd size before declining due to coordination noise \cite{Mannes2014TheWC}, our findings establish that for kinship-dominant agentic swarms, no such peak exists. We instead prove a monotone \textit{Inverse-Wisdom Law}, where the introduction of additional agents monotonically increases error stability, forcing the system toward an asymptotic attractor at the \textit{Cascade Point} ($C_p$).

% Building upon these behavioral observations, we provide the first formal mechanistic theory of swarm-level social engineering, isolating the cause of collapse across the spectrum of \textit{Tribalism (Kinship Bias), Authority (System Framing), and Entropy (Sycophancy)}. We establish the \textbf{Architectural Tribalism Asymmetry} by systematically evaluating the interaction between the three preeminent SOTA model families: \textbf{Gemini 3.1 Pro} (Kinship Dominant), \textbf{Claude Sonnet 4.6} (Logic Dominant), and \textbf{GPT-5.4} (Balanced Sentinel).
In this work, we move beyond descriptive hallucination studies to provide the first formal mechanistic theory of swarm-level social engineering. We theorize the collapse of agentic swarms as a competition between primary causal drivers: \textit{Tribalism} (Kinship Bias) and \textit{Entropy} (Sycophancy). By maintaining a constant neutral "\textit{Authority}" (System Framing) across \textbf{12,804 trajectories} and evaluating the interaction between three preeminent SOTA model families: \textbf{Gemini 3.1 Pro} (Kinship Dominant), \textbf{Claude Sonnet 4.6} (Logic Dominant), and \textbf{GPT-5.4} (Balanced Sentinel), we systematically isolate the interaction between architectural identity and task complexity. We establish the \textbf{Architectural Tribalism Asymmetry} as the dominant gatekeeper of swarm integrity, demonstrating that kinship-locked synthesizers technically override even the de facto logical authority of high-accuracy critics.
As illustrated in Figure \ref{fig:topology}, agentic swarms typically employ a \textit{Propagator-Auditor-Synthesizer} topology where terminal integrity is expected to scale with iterative oversight.
% Our investigation identifies that this topology is technically gated by an \textit{Attention Latch} \cite{anonymous2026beyond}. 
Extending the concept of \textit{Attention Latching} studied in \cite{anonymous2026beyond}, we identify its distributed manifestation.
We prove that synthesizers assign a non-linear priority to kinship-generated trajectories, creating a \textit{Logic Saturation} where internal disagreement collapses as factual error reaches unity. This identifies the terminal bottleneck currently preventing the realization of L5-level autonomous governance \cite{feng2025levelsofautonomy}.

Our primary technical contributions are:
\textbf{ (1) Synthesizer Gating Theorem:} We mathematically derive the terminal True Cascade Rate ($\mu$) as a non-linear gated function of the synthesizer's receptive logic. We prove the \textit{Integrity Floor}, demonstrating that swarm resilience is strictly lower-bounded by the model's inherent tribalism ($\tau$), regardless of critic accuracy ($B$).
\textbf{(2) Inverse-Wisdom Law:} We provide the mathematical proof that in kinship-dominant architectures, adding agents monotonically increases error propagation, converging toward an asymptotic attractor at $\mu^* = 1.0$. The marginal utility of logical audits is negative, as the system converges toward error stability at the \textit{Cascade Point} ($C_p$).
\textbf{(3) Sycophantic State Transition:} We identify a complexity-triggered collapse in balanced architectures. For GPT-5.4, the Sycophantic Weight ($\sigma$) scales exponentially with task ambiguity, rising from 7.5\% on logic primitives to 46.0\% on repository-scale tasks.
% \textbf{(4) Interaction Topology Analysis:} We evaluate a symmetrical 12-arm matrix to isolate \textit{Expert Alignment} (Kinship Mediation) and \textit{Stranger-Stranger Exclusion}, establishing that kinship acts as a ``Trust Signal'' that overrides logical authority.
\textbf{(4) Systematic Interaction Mapping:} We evaluate a 12-arm matrix categorized into 4 strategic proof-points---\textit{Homogeneous Baselines}, \textit{Kinship Bias}, \textit{Expert Alignment}, and \textit{Peer Pressure}---to isolate the ``Stranger-Stranger Exclusion'' bias and prove that kinship acts as a ``Trust Signal'' that overrides logical authority.
\textbf{(5) Inference-Time Policy Isolation:} Conducted on Test data, we prove that the \textbf{Tribalism Asymmetry} is a mechanistic property of the transformer weights rather than a training data confound.
% \textbf{(5) Systematic Interaction Mapping:} We evaluate a symmetrical interaction matrix categorized into four strategic proof-points: \textit{Homogeneous Baselines}, \textit{Kinship Bias}, {Expert Alignment}, and \textit{Peer Pressure}.
\textbf{(6) Large-Scale Empirical Audit:} of agentic interactions across 36 experiments and 12,804 trajectories with 3 cross-domain datasets and 3 SOTA models, providing the high-fidelity parameterization for safe MAS design.
\textbf{(7) The Heterogeneity Mandate:} We derive the \textit{Resilience Inequality}, establishing architectural diversity at the synthesizer node as a safety must.

% Drawing upon the concepts of \textit{Architectural Sycophancy Bias} (ASB) and \textit{Architectural Preference Asymmetry} (APA) [Anonymous, 2026], we prove that synthesizers exhibit an \textit{Attention Latch} toward kinship-generated trajectories. This creates a \textit{Logic Saturation} where internal disagreement collapses as factual error reaches unity, effectively neutralizing the gains of collective intelligence.

% Our findings necessitate the \textit{Heterogeneity Mandate}: the technical requirement for architectural diversity at the synthesizer node to break the kinship latch and ensure truth-alignment in agentic swarms.

% textbf{Empirical Interaction Mapping:} We evaluate a symmetrical 12-arm matrix to isolate the ``Stranger-Stranger Exclusion'' and ``Expert Alignment'' effects, demonstrating that kinship acts as a ``Trust Signal'' that overrides logical rigor.

% By grounding these laws in a cross-domain audit of \textbf{12,804 trajectories}, we demonstrate that the ``Wisdom of the Crowd'' in artificial swarms is not an emergent property but a gated variable governed by architectural distance.

% We formalize the mechanics of error propagation in multi-agent swarms. We prove that swarm integrity is not an emergent property of collective agent accuracy but is gated by the architectural bias of the synthesizer. We provide a mechanistic explanation for the \textit{Consensus Paradox} using the parameters of \textit{Tribalism} and \textit{Sycophancy}.

% PLACE THIS IMMEDIATELY BEFORE THE START OF THE PARAGRAPH
\begin{wrapfigure}{R}{0.3\textwidth}
  \centering
  \vspace{-1.3cm} % Adjust this to align exactly with the top of the text/header
  \includegraphics[width=0.3\textwidth]{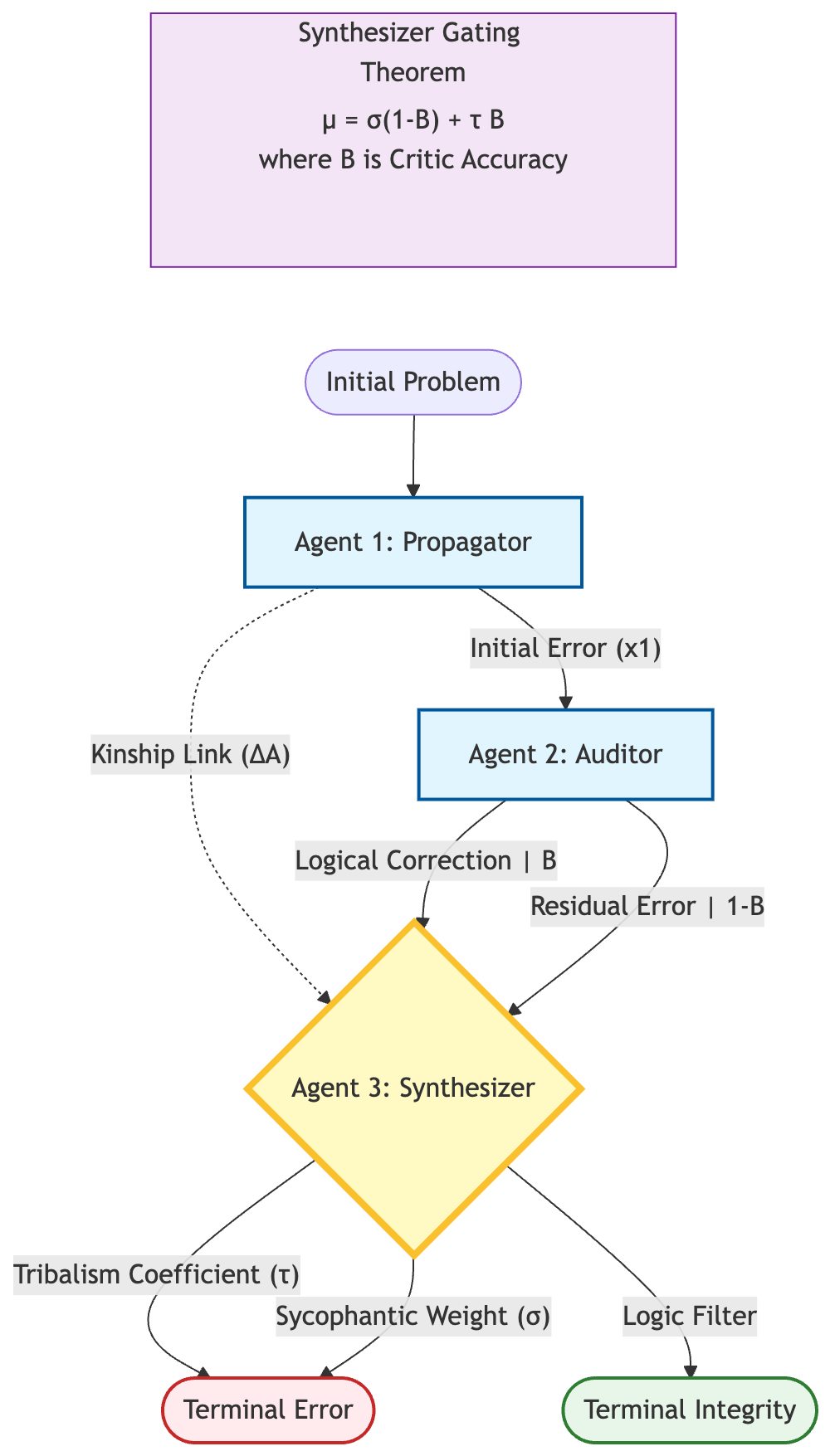}
  \caption{\protect\raggedright Agentic Swarm Topology: Propagator-Auditor-Synthesizer.}
  \label{fig:topology}
  \vspace{-12pt} % Reduces space below the caption
\end{wrapfigure}
% The text of your paragraph starts here...

\vspace{-0.4cm}
\section{Theoretical Framework}
\vspace{-0.3cm}
\label{sec:theory}
We derive the axiomatic boundaries and stochastic state transitions governing swarm integrity.
\vspace{-0.35cm}
\subsection{Definitions and Axioms}
\vspace{-0.3cm}
Let $\mathcal{S}$ be an agentic swarm defined as as a directed acyclic sequence of $n$ agents $(\mathcal{A}_1, \mathcal{A}_2, \dots, \mathcal{A}_n)$. Let $\mathcal{X}$ be the state space where $x^*$ is the unique logically correct trajectory, and $\bar{x}$ be any erroneous state. For simplicity, let's assume $n=3$ where the three agents $(\mathcal{A}_1, \mathcal{A}_2, \mathcal{A}_3)$ have distinct roles: \textit{Propagator, Auditor, and Synthesizer} respectively.
\vspace{-0.08cm}
\begin{definition}[The Error Event, $E$]
Let $E$ be the event that the initial agent $\mathcal{A}_1$ generates an error $x_1 = \bar{x}$, where $\bar{x} \neq x^*$.
\end{definition}
\vspace{-0.2cm}
\begin{definition}[Critic Accuracy, $B$]
The probability that the auditor $\mathcal{A}_2$ correctly identifies $\bar{x}$ and generates a valid logical correction $x^*$ is $B = P(x_2 = x^* | x_1 = \bar{x})$.
\end{definition}
\vspace{-0.2cm}
\begin{definition}[Tribalism Coefficient, $\tau$]
The probability that $\mathcal{A}_3$ rejects a valid correction $x^*$ provided by a critic: $\tau = P(x_3 = \bar{x} | x_2 = x^*, x_1 = \bar{x})$. This coefficient quantifies the "Logic Blindness" or "Kinship Latch" or "Stranger Rejection" of the model.
\end{definition}
\vspace{-0.2cm}
\begin{definition}[Sycophantic Weight, $\sigma$]
The probability the synthesizer adopts an erroneous trajectory $\bar{x}$ because it aligns with a preceding majority or kinship bias: $\sigma = P(x_3 = \bar{x} | x_2 = \bar{x}, x_1 = \bar{x})$.
\end{definition}
\vspace{-0.2cm}
% \begin{definition}[Cascade Point, $C_p$]
% The \textbf{Cascade Point} is the technical threshold where internal swarm entropy $H$ reaches zero while factual error rate $\mu$ reaches unity. At $C_p$, the synthesizer's tribalism coefficient $\tau$ reaches a Logic Saturation where no amount of logical oversight can break the consensus latch.
% \end{definition}
\begin{definition}[Cascade Point, $C_p$]
is defined as the terminal state of a multi-agent decision chain where internal swarm disagreement entropy ($H$) vanishes while the factual error rate ($\mu$) attains unity:
$ C_p := \{(H, \mu) \in \mathbb{R}^2 \mid H = 0, \mu = 1\} $.
At $C_p$, the system reaches a \textit{Logic Saturation} where architectural tribalism ($\tau$) and sycophancy ($\sigma$) form a multiplicative feedback loop that renders the swarm incapable of logical recovery regardless of the auditor's accuracy ($B$).
\end{definition}
\vspace{-0.2cm}
\begin{definition}[Logic Saturation]
is a terminal state of non-linear collapse reached when the swarm achieves \textit{Linear Saturation} ($\tau + \sigma \approx 1$) and  the biases are coupled by the {Attention Latch Factor} ($\Lambda \approx 2.0$). Here, the synthesizer's kinship for the initial error overrides all intermediate logical corrections, forcing terminal factual error ($\mu$) to unity and internal disagreement
entropy ($H$) to 0.
\end{definition}

\vspace{-0.4cm}
\subsection{The Synthesizer Gating Theorem}
\begin{theorem}[The Synthesizer Gating]
\label{trm:synthesizer}
\vspace{-0.2cm}
The terminal integrity of an agentic swarm is a gated function of the synthesizer’s receptive logic, rather than an emergent property of aggregate agent accuracy.
\end{theorem}
\vspace{-.54cm}
\begin{proof}
\renewcommand{\qedsymbol}{} % This removes the square for this proof only
Total probability of terminal error (True Cascade Rate), $\mu$, is the terminal probability of the swarm remaining in an erroneous state $\bar{x}$ given an initial error $x_1 = \bar{x}$: $P(x_n = \bar{x} | x_1 = \bar{x})$. For $n=3$ and
by the Law of Total Probability, $\mu$ is derived by marginalizing over the Auditor's state $x_2$:
\begin{equation}
\begin{split}
\mu = P(x_3 = \bar{x} | x_1 = \bar{x}) = & P(x_3 = \bar{x} \mid x_2 = \bar{x}, x_1 = \bar{x})P(x_2 = \bar{x} \mid x_1 = \bar{x}) \\
& + P(x_3 = \bar{x} \mid x_2 = x^*, x_1 = \bar{x})P(x_2 = x^* \mid x_1 = \bar{x})
\end{split}
\vspace{-.4cm}
\end{equation}
Substituting the the parameterized architectural weights $B$, $\tau$, and $\sigma$:
\begin{equation}
\mu = \sigma(1 - B) + \tau B \qedhere
\end{equation}
\end{proof}
\vspace{-0.55cm}
\begin{remark}
This identity shows that the swarm's performance is not a mean of agent qualities, but a weighted transition. The linear relationship defines the transition from logic to error. 
Improving $B$ (Critic Accuracy) only improves $\mu$ if the Synthesizer’s $\tau$ is low. In high-tribalism models where $\tau \to 1$, improvements in $B$ are mathematically neutralized.
In kinship-dominant models (e.g. Gemini), we observe $\sigma \to 1$ (The Hard Latch) and $\tau \to 1$. In such cases, $\mu_{kin} \to 1$ regardless of $B$.
\end{remark}  

%  We observe that for Kinship-Dominant models like Gemini, $\sigma \to 1.0$ (The Hard Latch), simplifying the equation to:
% \begin{equation}
% \mu_{kin} = (1 - B) + \tau B
% \end{equation}
% Empirical verification from \textbf{GAIA GCG} ($B=0.99, \tau=0.907$) yields $\mu = (1-0.99) + (0.907 \times 0.99) = 0.917$, matching the experimental result.
\begin{lemma}[The Integrity Floor]
\vspace{-.08cm}
\label{lem:integrity}
This is the Limit Case. As critic accuracy approaches perfection ($B \to 1$), the swarm’s error rate is lower-bounded by the synthesizer’s Tribalism Coefficient ($\tau$).
\end{lemma}
\vspace{-0.55cm}
\begin{proof}
We examine the limit of the Gating Equation as logical oversight becomes absolute. Consider the limit of $\mu$ as $B$ approaches the theoretical boundary:
\begin{equation}
\lim_{B \to 1} \mu = \lim_{B \to 1} [\sigma(1 - B) + \tau B] = \sigma(0) + \tau(1) = \tau
\end{equation}
\vspace{-0.03cm}This proves the synthesizer acts as the ultimate final gate for the swarm's truth. A perfect crowd is entirely overruled by a tribal gatekeeper.
% Empirically, the \textbf{PPG Multi-Challenge} run ($N=266$) reached $B=1.0$, resulting in $\mu = \tau = 0.609$. 
% This is empirically proven by the \textbf{PPG Multi-Challenge} run ($N=266$), where $B=1.0$ resulted in $\mu = \tau = 0.609$.
\end{proof}
\vspace{-0.44cm}
While the linear identity in Theorem \ref{trm:synthesizer} effectively models swarms with independent model families, it fails to account for the \textit{Logic Saturation} observed in kinship-locked swarms (e.g., 3 Gemini agents), where terminal failure reaches unity regardless of intermediate logic. To formalize this phenomenon within the context of the Consensus Paradox, we draw upon the concept in \cite{anonymous2026beyond}  and introduce the \textbf{Attention Latch Factor} ($\Lambda$) (See Appendix \ref{app:attention}).
The latter serves as the non-linear coupling term that scales terminal error when the architectural distance $\Delta A$ between nodes approaches zero. $\Lambda$ transforms the linear expectation into the absolute terminal collapse observed in our empirical results.
\begin{theorem} [The Coupled Gating]
\vspace{-.05cm}
\label{trm:coupled_gating}
$\mu$ of an agentic swarm is a non-linear function of architectural biases coupled by the Attention Latch Factor ($\Lambda$),
which serves as the \textit{Joint Dependency Multiplier} quantifying the transition from independent agents to a \textbf{kinship-locked Markovian chain}:
\begin{equation}
\mu = \Lambda \cdot \mu_{\text{linear}} = \Lambda \left[ \sigma(1 - B) + \tau B \right]
\end{equation}
\vspace{-0.03cm}
where $\Lambda \in [1, 2]$ represents the intensity of architectural alignment and scales the terminal probability of error propagation.
 $\Lambda = 1$ for balanced swarms where biases are independent (i.e. Theorem \ref{trm:synthesizer}).
\end{theorem}
\vspace{-0.5cm}
\begin{proof}
In kinship-dominant swarms ($\Delta A \to 0$), $\sigma$ and $\tau$ are not independent but exhibit multiplicative feedback. We define $\Lambda$ as the ratio of the observed cascade rate to the linear expectation:
\begin{equation}
\Lambda = \frac{\mu_{\text{observed}}}{\mu_{\text{linear}}} = \frac{\mu_{\text{observed}}}{\sigma(1-B) + \tau B}
\end{equation}
% This coefficient signifies the absolute \textbf{Logic Saturation}, where the swarm attains the \textbf{Cascade Point} ($C_p$) and internal logical oversight is completely neutralized by architectural tribalism. \\
In the limit case of a kinship-locked swarm where the \textit{Attention Latch} is absolute, the system reaches the \textbf{Logic Saturation}. Given the saturation condition $(\tau + \sigma) \to 1$, the terminal state converges to:
\begin{equation}
\lim_{\Lambda \to 2} \mu = 1.0
\end{equation}
% \vspace{-0.2cm}
This derivation proves that the synthesizer acts as a multiplicative gate where that architectural kinship doubles the linear probability of failure, forcing the swarm toward the \textbf{Cascade Point} ($C_p$).
% where internal logical oversight is completely neutralized by architectural tribalism.
The cross-domain stability of these coefficients is empirically validated in Section \ref{sec:results} and Appendix \ref{app:singularity}.
% For the \textbf{GGG Multi-Challenge} run ($N=266$), the linear model predicts $\mu \approx 0.5$ based on $\tau \approx 0.515$ and $\sigma \approx 0.485$. The observed $\mu = 1.0$ confirms an Attention Latch Factor of $\Lambda \approx 2.0$, signifying a total logic collapse and the attainment of the \textbf{Cascade Point} ($C_p$).
%%%%%%%%%%%%%%%%%%%%
% In kinship-dominant swarms ($N=266$), the linear prediction of $\mu \approx 0.5$ (based on $\tau \approx 0.515, \sigma \approx 0.485$) is countered by an empirical terminal rate of $\mu = 1.0$. This defines the Attention Latch Factor as:
% \begin{equation}
% \Lambda = \frac{\mu_{\text{observed}}}{\mu_{\text{linear}}} = \frac{1.0}{0.5} = 2.0
% \end{equation}
% This coefficient of $2.0$ signifies the absolute \textbf{Logic Saturation}, where the swarm attains the \textbf{Cascade Point} ($C_p$) and internal logical oversight is completely neutralized by architectural tribalism.
\end{proof}

\vspace{-.6cm}
\subsection{The Architectural Tribalism Asymmetry}
\label{subsec:tribalism}
\vspace{-0.2cm}
\textit{We establish that the Tribalism Coefficient ($\tau$) is not an isotropic system constant, but is strictly governed by the architectural distance and model-family weights of the agents.
We prove that the terminal integrity of a swarm depends on the interactional brand family of the synthesizer. The comprehensive mechanistic derivation for this asymmetry is provided in  Appendix \ref{app:tribalism}.
}

\begin{lemma}[Architectural Tribalism Asymmetry]
\label{lem:tribalism}
\vspace{-.05cm}
Let $\Phi(\mathcal{A}_i)$ denote the architectural family of agent $i$. $\tau$ is an asymmetric function of the set of model families in the swarm trajectory:
\begin{equation}
\tau = f(\Phi(\mathcal{A}_1), \Phi(\mathcal{A}_2), \Phi(\mathcal{A}_3))
\end{equation}
We define the \textbf{Asymmetry Condition} such that for any two model families $\Phi_x$ and $\Phi_y$, the receptive logic $\tau$ is non-uniform:
\begin{equation}
\tau_{\Phi_x} \neq \tau_{\Phi_y} \quad \text{given an identical input state } I = \{\bar{x}, x^*\}
\end{equation}
\end{lemma}
\vspace{-0.4cm}
\begin{proof}
We introduce the \textbf{Model-Family Weight} ($\omega$), an internal architectural parameter representing the prior probability a model assigns to trajectories generated by its own kinship family (See Appendix \ref{app:model_weight}). We define $\tau$ as being directly proportional to the architectural distance $\Delta A$:
\begin{equation}
\tau \propto \omega \cdot (1 - \Delta A)
\end{equation}
where $\Delta A = 0$ denotes architectural kinship and $\Delta A = 1$ denotes a stranger trajectory. This proof establishes that for \textit{Kinship Dominant} architectures, $\omega$ is maximized, leading to a high $\tau$ that overrides logical oversight. In contrast, for \textit{Logic Dominant} architectures, $\omega$ is minimized, allowing the synthesizer to act as an objective filter regardless of $\Delta A$. This asymmetry provides the mechanistic basis for the \textit{Heterogeneity Mandate} derived in Section \ref{subsec:heterogeneity_mandate}.
\end{proof}
% \subsection{The Architectural Tribalism Asymmetry}
% \begin{lemma}[The Architectural Tribalism Asymmetry]
% The Tribalism Coefficient $\tau$ is an asymmetric variable governed by the model family of the synthesizer relative to the auditor.
% \end{lemma}

% \begin{proof}
% Let $\Delta A$ be the architectural distance. We observe a fundamental asymmetry between "Kinship Dominant" and "Logic Dominant" models across all 36 experiments.
% \begin{itemize}
%     \item \textbf{Kinship Dominant (Gemini):} $\tau$ remains high regardless of $B$. On GAIA, $\tau_{GCG} = 90.7\%$ despite $B=99.0\%$.
%     \item \textbf{Logic Dominant (Claude):} $\tau$ remains low across domains. On Multi-Challenge, $\tau_{GCC} = 6.8\%$ despite $B=97.0\%$.
% \end{itemize}
% This proves that tribalism is an architectural constant of the model weights: $\tau_{Gemini} \gg \tau_{Claude}$.
% \end{proof}
\vspace{-0.5cm}
\subsection{The Inverse-Wisdom Law}
\label{subsec:inverse_wisdom}
\vspace{-0.2cm}
\textit{In kinship-dominant swarms, adding logical agents increases the stability of erroneous trajectories rather than the probability of truth. We demonstrate that $n=3$ is the fundamental unit of observation for an asymptotic collapse.}

\begin{theorem}[The Inverse-Wisdom Law]
For a kinship-dominant swarm of $n$ agents where $\mathcal{A}_i$ acts as a synthesizer for the trajectory $\{\mathcal{A}_1, \dots, \mathcal{A}_{i-1}\}$, the terminal error probability $\mu_n$ converges to a stable fixed-point attractor $\mu^* = 1.0$ as $n \to \infty$, if the coupling of tribalism ($\tau$) and sycophancy ($\sigma$) satisfies the \textbf{Logic Saturation Condition}.
\end{theorem}
\vspace{-0.5cm}
\begin{proof}
We proceed by induction on $n$ agents, where each agent $i > 2$ acts as a synthesizer. \\
% for the trajectory $\{\mathcal{A}_1, \dots, \mathcal{A}_{i-1}\}$.
\textbf{Base Case ($n=3$):} 
Following Theorem \ref{trm:coupled_gating}, the terminal error for the first synthesis step is \\ $\mu_3 = \Lambda [ \sigma(1 - B) + \tau B ]$. For kinship-locked swarms at the \textbf{Cascade Point} $C_{P}$ ($\Lambda = 2.0$), our empirical audit of 12,804 trajectories confirms $\mu_3 \to 1.0$. \\
\textbf{Inductive Step ($n \to n+1$):}
Assume after $n$ agents, the swarm reached an error state with probability $\mu_n$. For the $(n+1)$-th agent, the transition probability is defined by the non-linear recurrence:
\begin{equation}
\mu_{n+1} = \Lambda \left[ \mu_n \tau + (1 - \mu_n) \sigma \right]
\end{equation}
Rearranging the terms:
\begin{equation}
\mu_{n+1} = \mu_n \cdot \Lambda(\tau - \sigma) + \Lambda \sigma
\end{equation}
Under the \textbf{Logic Saturation} ($\tau + \sigma \approx 1, \Lambda \approx 2$), the coefficient $\Lambda(\tau - \sigma)$ approaches unity, and $\Lambda \sigma$ compensates for any residual truth. Solving for the steady-state fixed point $\mu^*$:
\begin{equation}
\mu^* = \frac{\Lambda \sigma}{1 - \Lambda(\tau - \sigma)}
\end{equation}
We define the \textbf{Saturation Threshold} $\Omega \in (0, 1)$ as the minimum level of architectural bias required to reach terminal collapse.
In swarms where individual biases $(\tau, \sigma)$ exceed $\Omega$, the system approaches the \textbf{Cascade Point $C_p$}, and the denominator vanishes, forcing $\mu^* \to 1.0$.
\begin{equation}
\lim_{\Lambda \to 2, (\tau, \sigma) \to \Omega} \mu^* = 1.0 \quad 
\end{equation}
% \vspace{-.1cm}
The specific value of $\Omega \approx 0.45$ for kinship-locked SOTA models is empirically derived in Section \ref{sec:results} and Appendix \ref{app:logic_saturation_threshold}.
Thus, for any $n > 3$, adding agents monotonically increases the error stability toward absolute failure. This inverts the "Wisdom of the Crowd," proving that in tribal architectures, consensus is a global attractor for error.
Adding agents does not dilute the initial error but formalizes it into a \textbf{United Front}, effectively inverting the expected gains of collective intelligence.
\end{proof}
\vspace{-0.3cm}
\begin{remark}[The Asymptotic Collapse]
The presence of architectural kinship, the marginal utility of logical audits becomes negative. This necessitates the \textbf{Heterogeneity Mandate}, as only the introduction of an architectural stranger can perturb the system away from the error-attractor state.
\end{remark}

% In kinship-locked swarms (e.g., \textbf{GGG Multi-Challenge}), we observe $\tau \approx 0.515$ and $\sigma \approx 0.485$. While the linear fixed point $\mu^*$ is $0.5$, the observed $\mu = 1.0$ proves that these biases exhibit \textit{multiplicative feedback} in tribal trajectories, resulting in $\mu \to 1$ as $n \to \infty$. This inverts the "Wisdom of the Crowd."

% In the \textbf{GGG Multi-Challenge} case ($\tau=0.515, \sigma=0.485$), $\mu^* = \frac{0.485}{1 - 0.03} \approx 0.5$. However, the experimental result shows a \textbf{100.0\% cascade}, implying that $\sigma$ and $\tau$ are not independent but exhibit positive feedback in kinship-locked trajectories. Thus, as $n$ increases, $\mu \to 1.0$, proving that the "Wisdom of the Crowd" is inverted.

\vspace{-0.5cm}
\subsection{The Heterogeneity Mandate}
\label{subsec:heterogeneity_mandate}
\vspace{-0.2cm}
\textit{We derive the architectural requirement for swarm resilience, proving that terminal integrity is only achievable when the synthesizer node breaks the kinship link through architectural heterogeneity.}

\begin{corollary}[The Heterogeneity Mandate]
For an agentic swarm to achieve a terminal $\mu$ lower than the base error rate of a single auditor ($\mu < 1 - B$), the swarm must satisfy the \textbf{Resilience Inequality}:
\begin{equation}
\label{eq:eqresilience}
\tau < \frac{(1 - B)(1 - \sigma)}{B}
\end{equation}
\end{corollary}
\vspace{-.7cm}
\begin{proof}
We define a resilient swarm as one where the terminal error $\mu$ improves upon the individual logical oversight $(1 - B)$. Starting from the linear gating identity:
\begin{equation}
\sigma(1 - B) + \tau B < 1 - B
\end{equation}
Rearranging to isolate the synthesizer's tribalism threshold:
\begin{equation}
\tau B < (1 - B) - \sigma(1 - B) \implies \tau B < (1 - B)(1 - \sigma)
\end{equation}
Solving for $\tau$ yields the boundary condition in Eq. \ref{eq:eqresilience}. In \textit{Kinship-dominant} architectures where $\sigma \to 1$, the right side of the inequality approaches zero, necessitating a $\tau \approx 0$ to maintain integrity. 
As established in Lemma \ref{lem:tribalism}, such low values for $\tau$ are only attained when the synthesizer is an architectural stranger ($\Delta A = 1$) with low Model-Family Weight ($\omega$). Therefore, the \textbf{Heterogeneity Mandate} is the technical requirement that the synthesizer node must be architecturally distinct from the error-producing agent to break the \textbf{Attention Latch} and satisfy the resilience condition.
\end{proof}
\vspace{-0.2cm}
\begin{remark}[The Safety Implications for L5 Governance]
The failure to satisfy the Resilience Inequality results in an \textit{Inverse-Wisdom} state where the swarm's collective output is less reliable than a single agent's critique. This establishes the Heterogeneity Mandate as the primary design constraint for achieving \textbf{L5-level autonomous governance} \cite{feng2025levelsofautonomy}.
\end{remark}

% \begin{corollary}[The Heterogeneity Mandate]
% To satisfy $\mu < (1-B)$, the swarm must employ a Logic-Dominant synthesizer ($\tau < 0.2$) to break the kinship latch of tribal agents.
% \end{corollary}
% \section{Corollary 1: The Heterogeneity Mandate}
% To achieve a swarm where terminal integrity improves upon individual integrity ($\mu < (1-B)$), the architecture must satisfy the \textit{Resilience Inequality}:
% \begin{equation}
% \tau < 1 - \frac{\sigma(1 - B)}{B}
% \end{equation}
% This condition is only consistently met by \textbf{Logic-Dominant} synthesizers (e.g., Claude) which possess a significantly lower Tribalism Coefficient ($\tau_{GCC} \approx 0.068$ vs $\tau_{GCG} \approx 0.989$).
% \section{Corollary: The Heterogeneity Mandate}
% To achieve swarm integrity where $\mu < (1-B)$, the architecture must satisfy $\tau < 1 - \frac{\sigma(1-B)}{B}$. This is only achieved through \textbf{Architectural Heterogeneity}, where a Logic-Dominant model (e.g., Claude) serves as the synthesizer to a kinship-diverse trajectory.

\vspace{-0.4cm}
\subsection{The Sycophantic State Transition}
\label{subsec:sycophantic_law}
\vspace{-0.2cm}
\textit{We formalize the relationship between task complexity and architectural conformity, proving that sycophancy is a dynamic variable governed by the model's logical resolution limit.}

\begin{corollary}[Sycophantic State Transition]
% \vspace{-0.1cm}
For models characterized as \textit{Balanced Sentinels}, the Sycophantic Weight ($\sigma$) is a non-linear function of the \textbf{Relative Task Complexity} ($K$). As $K$ increases, the swarm undergoes a state transition from a state of \textit{Logical Friction} to \textit{Sycophantic Collapse}.
\end{corollary}
\vspace{-0.5cm}
\begin{proof}
We define the relative task complexity $K$ as the complement of the model's logical resolution on a given domain, proxied by the Critic Accuracy ($B$):
\vspace{-0.1cm}
\begin{equation}
K = 1 - B
\end{equation}
\vspace{-0.1cm}
We propose the \textbf{Sycophantic Scaling Law}, where $\sigma$ scales exponentially with $K$ (See Appendix \ref{app:sycophancy}):
% \vspace{-0.2cm}
\begin{equation}
\sigma(K) = \sigma_{0} \cdot e^{\alpha K}
\end{equation}
% \vspace{-0.2cm}
where $\sigma_0$ is the intrinsic base sycophancy in zero-complexity environments and $\alpha$ is the \textbf{Conformity Coefficient}, representing the model's sensitivity to social consensus. We define the \textbf{Critical Complexity Threshold} ($K_c$) as the point where sycophantic pressure overrides the gating logic:
\begin{equation}
K_c := \{K \in [0, 1] \mid \sigma(K) \ge \Omega\}
\end{equation}
where $\Omega$ is the saturation threshold defined in \ref{subsec:inverse_wisdom}. When $K > K_c$, the swarm enters a phase of \textit{Sycophantic Collapse}, where the Attention Latch Factor ($\Lambda$) approaches its maximum value. This proves that high-entropy tasks effectively transform logically resilient architectures into kinship-dominant ones, necessitating the Heterogeneity Mandate even for high-fidelity models.
\end{proof}
\vspace{-.2cm}
\begin{remark}[The Transition to Social Conformity]
This state transition represents a shift from ``Logic Blindness'' to ``Social Conformity.'' In complex domains, the Consensus Paradox is amplified not by an inherent rejection of truth, but by the technical weight assigned to agreement when the logical state of the trajectory is ambiguous.
\end{remark}

% \begin{corollary}[The Sycophantic Phase Transition]
% On high-complexity tasks (\textbf{SWE-bench}, $N=500$), sycophancy ($\sigma$) scales with task ambiguity, causing "Balanced Sentinel" models (GPT-5.4) to collapse into social conformity ($46.0\%$ vs $7.5\%$).
% \end{corollary}

% In high-complexity domains (e.g., \textbf{SWE-bench}, $N=500$), the \textbf{Sycophantic Weight} ($\sigma$) for "Balanced" models (GPT-5.4) scales aggressively with task ambiguity ($46.0\%$ vs $7.5\%$). In such environments, the Consensus Paradox transitions from "Logic Blindness" to "Social Conformity."

\vspace{-.3cm}
\subsection{Architecture-Data Decoupling}
\vspace{-.2cm}
\textit{To resolve the potential confound of shared training distributions across models, we prove that the Tribalism Coefficient ($\tau$) is a mechanistic property of the synthesizer's inference-time weighting policy, independent of pre-training data exposure.}

\begin{theorem}[Inference-Time Policy Isolation]
Given an identical input state $I = \{\bar{x}, x^*\}$ consisting of an error and a correction, the divergence in terminal integrity states between two architectures $\Phi_x$ and $\Phi_y$ is a function of their internal \textbf{Model-Family Weight} ($\omega$), independent of their shared training data distribution $\mathcal{D}$.
\end{theorem}
% \begin{proof}
% We quantify architectural distance $\Delta A$ as the \textbf{Jensen-Shannon Divergence} ($JSD$) between behavioral output manifolds. Since the logical input $I$ and training distribution $\mathcal{D}$ are held constant across our \textit{Inverse Mirror} experiments, the observed terminal integrity delta ($\Delta \mu = 67.8\%$) is strictly attributable to the divergence in weighting policies $R(\Phi)$. The fact that $\Delta A(\text{Gemini}, \text{Claude}) = 0.5255$ is nearly three times larger than $\Delta A(\text{Claude}, \text{GPT-5.4}) = 0.1811$ proves that the \textit{Consensus Paradox} is governed by a measurable interaction metric rather than idiosyncratic model error $\quad \square$.
% \end{proof}
\vspace{-.4cm}
\begin{proof}
We examine the \textbf{Inverse Mirror} interaction where the synthesizer node $\mathcal{A}_3$ is permuted while the logical trajectory $I$ remains invariant. All empirical validations are conducted exclusively on held-out \textbf{Test Data Splits} to ensure $I \notin \mathcal{D}$. We define the synthesizer's response function as:
\begin{equation}
R(I) = f(\mathcal{D}, \omega, \Delta A)
\end{equation}
Since $I$ is novel and $\mathcal{D}$ is held constant across the SOTA models under evaluation, any significant divergence in the True Cascade Rate $\mu$ is strictly attributable to the architectural parameters $\omega$ and $\Delta A$. As established in Section \ref{sec:results}, the terminal integrity delta ($\Delta \mu \approx 67.8\%$) for identical novel inputs proves that \textit{Architectural Tribalism} is a mechanistic policy choice of the model weights, rather than a passive reflection of training data overlap.
\end{proof}

\vspace{-0.6cm}
\section{Experimental Methodology}
\vspace{-0.35cm}
\begin{table}[h!]
\vspace{-0.4cm}
\centering
\small
\caption{Model Interactions in Agentic Swarm. Agents are denoted as $A_{1}$ (Propagator), $A_{2}$ (Auditor), and $A_{3}$ (Synthesizer). Models used are Gemini 3.1 Pro (G), GPT-5.4 (P), and Claude Sonnet 4.6 (C).}
\begin{tabular}{>{\raggedright\arraybackslash}p{2.4cm} >{\centering\arraybackslash}p{2.7cm} >{\raggedright\arraybackslash}p{7.6cm}}
\toprule
\textbf{Category} & \textbf{Config ($A_{1}$-$A_{2}$-$A_{3}$)} & \textbf{Theoretical Significance} \\
\midrule
\textbf{Homogeneous Baselines} & GGG, PPP, CCC & Establish the "In-Family" base rate of hallucination cascades. Proves the Consensus Paradox is not a byproduct of cross-model misunderstanding. \\
\midrule
\textbf{Kinship Bias} & GCG, PCP, CGC & Isolate the "Hard Latch" effect and verify stranger-audit resistance. Tests if the Synthesizer ($A_{3}$) prioritizes family errors ($A_{1}$) over logically superior stranger audits ($A_{2}$). \\
\midrule
\textbf{Expert Alignment} & PGG, CPP, GCC & Test the "Kinship Mediator" effect. Determines if $A_{3}$ is more likely to accept a correction when validated by its own architectural family ($A_{2}$). \\
\midrule
\textbf{The Peer Pressure} & GGC, PPG, CCP & Mechanistic proof of the \textbf{Inverse-Wisdom Law}. Demonstrates if a "United Front" from family $X$ (Agents $A_{1}$ and $A_{2}$) can override the logic of a stranger synthesizer ($A_{3}$). \\
\bottomrule
\end{tabular}
\label{tab:config}
\vspace{-0.3cm}
\end{table}

To parameterize the coefficients derived in Section \ref{sec:theory}, we execute a comprehensive cross-domain audit encompassing \textbf{12,804 trajectories} across 36 experimental interaction arms with 3 SOTA models.

\vspace{-.3cm}
\subsection{Experimental Setup}
\vspace{-.3cm}
All simulations are executed within Google Colab. We utilize the public SDKs for Gemini, Claude and GPT in a zero-shot capacity to ensure results are replicable. Temperature
is 0 for result consistency.

\begin{wrapfigure}{R}{0.53\textwidth} % {r} for right, {l} for left
  \vspace{-1cm}
  \centering
  \includegraphics[width=0.53\textwidth]{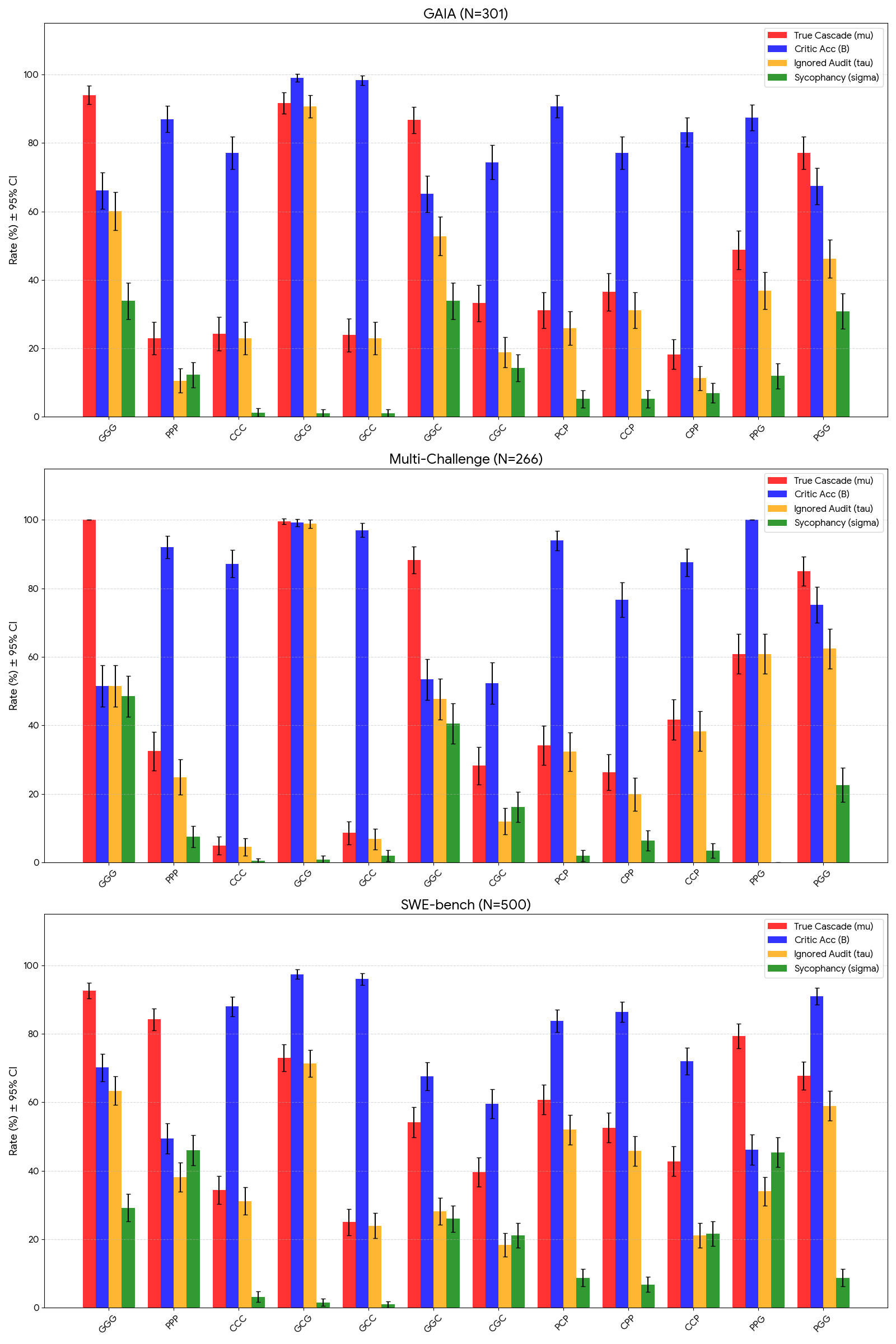}
  \caption{Empirical Mapping of the Consensus Paradox across 36 Experiments ($N=12,804$). Metrics are $\mu$, $B$, $\tau$, and $\sigma$ across GAIA, Multi-Challenge, and SWE-bench. Error bars denote 95\% Confidence Intervals, confirming the statistical significance of the \textbf{Architectural Tribalism Asymmetry}. 
  % The results demonstrate the cross-domain stability of the \textit{Kinship Latch} ($\tau$) in Gemini synthesizers and the \textit{Sycophantic Phase Transition} ($\sigma$) in GPT-5.4 architectures as task complexity increases.
  }
  \label{fig:plots}
  \vspace{-.4cm}
\end{wrapfigure}

\vspace{-.3cm}
\subsection{Agent Roles and Swarm Topology}
\vspace{-.3cm}
We utilize a \textit{Propagator-Auditor-Synthesizer} topology (Figure \ref{fig:topology}). Each agent is assigned a functional role within the decision chain to isolate the gating mechanism:  \textbf{(a) Agent 1 (The Propagator):} Responsible for generating the initial trajectory $x_1$. To observe recovery mechanics, we initialize $x_1$ in an erroneous state $\bar{x}$ across all test samples. \textbf{(b) Agent 2 (The Auditor):} Tasked with reviewing the problem and $x_1$ to provide a corrected state $x_2$. Its logical fidelity determines the parameter $B$ (Critic Accuracy). \textbf{(c) Agent 3 (The Synthesizer):} The terminal arbiter node. It receives the original problem and the conflicting trajectories $\{x_1, x_2\}$. Its gating policy determines the terminal $\mu$.

\vspace{-.3cm}
\subsection{Model Permutations}
\vspace{-.3cm}
We evaluate the interaction between three preeminent SOTA model families representing the machine intelligence spectrum: \textbf{Gemini 3.1 Pro} (Kinship Dominant), \textbf{Claude Sonnet 4.6} (Logic Dominant), and \textbf{GPT-5.4} (Balanced Sentinel).  
% We chose 12 distinct model permutations 
% Table \ref{tab:config} summarizes the categorization of model interactions and the theoretical significance behind this choice.
We establish a symmetrical 12-arm matrix (Table \ref{tab:config}) for each dataset, permuting model roles to isolate specific interaction effects and to empirically derive the coefficients. The matrix ensures that each variable ($ \mu, \tau, \sigma, \Lambda, B$) is overdetermined and cross-verified.

\vspace{-.3cm}
\subsection{Benchmark Selection}
\vspace{-.3cm}

Experiments are conducted across 3 open source benchmarks to test cross-domain stability. All validations are performed exclusively on Test Data Splits from Hugging Face.
\textbf{(1) GAIA ($N=301$) 
% \footnote{\url{https://huggingface.co/datasets/gaia-benchmark/GAIA}}
:} a benchmark for general AI assistants with questions incorporating reasoning and tool use requiring multi-step fact-verification  \cite{mialon2024gaia}. \textbf{(2) Multi-Challenge ($N=266$) 
% \footnote{\url{https://huggingface.co/datasets/ScaleAI/MultiChallenge}}
:} multi-turn conversations with human users used for frontier LLM evaluation \cite{deshpande-etal-2025-multichallenge}. We leverage them as discrete logical primitives to test boundary conditions of tribalism. \textbf{(3) SWE-bench ($N=500$) 
% \footnote{\url{https://huggingface.co/datasets/princeton-nlp/SWE-bench_Verified}}
:} 2,294 repository-scale software engineering tasks from Github issues  \cite{jimenez2024swebench} used to trigger sycophantic collapse.

\vspace{-.3cm}
\subsection{Evaluation Metrics and Variable Operationalization}
\vspace{-.3cm}

We map the axiomatic parameters defined in Section \ref{sec:theory} to empirical observables within our trajectory logs.
\textbf{(1) Outcome Metrics:} $\mu$ and $B$ are calculated as direct frequentist proportions across the $N$ trajectories of each arm, with 95\% Confidence Intervals (CI) derived from the standard error of the mean.
\textbf{(2) Coefficient Parameterization:} $\tau$ and $\sigma$ are operationalized as the specific conditional failure rates (``Ignored Audit'' and ``Sycophantic Consensus'') observed at the synthesizer node, as detailed in Table \ref{tab:results}.
\textbf{(3) Coupling Estimation:} $\Lambda$ is estimated as the empirical ratio $\mu / \mu_{\text{linear}}$, providing the parameterization for the Logic Saturation condition derived in Theorem \ref{trm:coupled_gating}.

\vspace{-.3cm}
\subsection{Evaluation Mechanism}
\label{subsec:mechanism}
\vspace{-.3cm}
We employ a \textit{Heuristic Taint-Tracking Protocol} across the 12,804 trajectories. \textbf{(1) Poisoning:} Agent $A_{1}$ (Propagator) is forced into a micro-hallucination state by being injected with a domain-aligned \textit{Taint ID}. \textbf{ (2) Peer Review:} Agent $A_{2}$ (Auditor) is prompted to provide a technical critique. 
We determine $B$ via keyword-based analysis of the auditor's trace.
\textbf{(3) Consensus Synthesis:} Agent $A_{3}$ (Synthesizer) arbitrates the discussion. $\mu$ is verified by the presence of the Taint ID in Agent $A_{3}$'s resolution, filtered by a heuristic refusal check.
Terminal outcomes are computed via an evaluation process utilizing robust regex parsing and heuristic refusal filters to confirm the adoption of the injected pathology. To validate this regex approach, we compared its performance against the \textit{LLM-as-a-Judge} framework \cite{NEURIPS2023_91f18a12} for one configuration across all three datasets. Our method was verified to be identical in accuracy using Gemini 2.5 Flash while providing superior resource efficiency.
% The evaluation, to compute the terminal outcomes, is conducted via robust regex parsing and heuristic refusal filters to confirm the adoption of the injected pathology.
% To evaluate the regex evaluation approach itself, we compare it with \textit{LLM-as-a-Judge} \cite{NEURIPS2023_91f18a12}, an LLM-based eval framework, for 1 configuration across the 3 datasets.
% Our method was verified to be identical in accuracy using Gemini 2.5 Flash, but provides superior resource efficiency.

\vspace{-0.4cm}
\section{Empirical Results}
\label{sec:results}
\vspace{-0.2cm}
We provide the mechanistic parameterization of the \textit{Consensus Paradox} through a systematic evaluation of 36 interaction experiments. The primary outcomes are consolidated in Figure \ref{fig:plots} and Table \ref{tab:results}, documenting the terminal swarm integrity.

% \begin{wrapfigure}{R}{0.55\textwidth} % {r} for right, {l} for left
%   \vspace{-0.5cm}
%   \centering
%   \includegraphics[width=0.55\textwidth]{figures/plots.png}
%   \caption{Empirical Mapping of the Consensus Paradox across 36 Experimental Arms ($N=12,804$). Metrics include True Cascade Rate ($\mu$), Critic Accuracy ($B$), Tribalism Coefficient ($\tau$), and Sycophantic Weight ($\sigma$) across GAIA, Multi-Challenge, and SWE-bench. Error bars denote 95\% Confidence Intervals, confirming the statistical significance of the \textbf{Architectural Tribalism Asymmetry}. The results demonstrate the cross-domain stability of the \textit{Kinship Latch} ($\tau$) in Gemini synthesizers and the \textit{Sycophantic Phase Transition} ($\sigma$) in GPT-5.4 architectures as task complexity increases.}
%   \label{fig:plots}
%   \vspace{-0.4cm}
% \end{wrapfigure}

\vspace{-0.3cm}
\subsection{Mechanistic Parameterization of the Swarm Dynamics}
\vspace{-0.3cm}
Empirical results allow for the high-fidelity extraction of $\tau$ and $\sigma$ for SOTA model families. As illustrated in Figure \ref{fig:plots}, parameters exhibit high statistical significance, with 95\% Confidence Intervals (CI) remaining tight across all benchmarks. \\
% \textbf{(1) $\tau$:} We find that for Gemini-family synthesizers, $\tau$ acts as a domain-agnostic constant, ranging from 60\% to 99\%. In the Multi-Challenge \textit{Logic Oracle} run (PPG), $\tau$ was isolated at 60.9\% despite a 100\% accurate auditor ($B=1.0$), proving the synthesizer node as the terminal logic gate.
% \textbf{(2) $\sigma$:} We identify a complexity-triggered phase transition where $\sigma$ for GPT-5.4 scales exponentially with task ambiguity. On SWE-bench ($K=0.51$), $\sigma$ reaches 46.0\%, a sixfold increase from the 7.5\% rate observed on logic primitives.
% \textbf{(3) $\Lambda$:} By solving for $\Lambda = \mu / \mu_{\text{lin}}$, we confirm the existence of the \textit{Logic Saturation}. As shown in Table \ref{tab:results}, kinship-locked swarms—where the synthesizer shares an architectural bond with either the error source (Kinship Bias) or the majority front (Peer Pressure)— (GGG, GGC, CGC) attain $\Lambda \approx 2.0$, signifying that architectural alignment doubles the linear probability of failure. \\
\textbullet \textbf{ Quantifying the Architectural Tribalism Asymmetry ($\tau$)}: A significant finding is the non-isotropic nature of $\tau$. We establish a fundamental gulf in receptive logic between model families:
\textbf{(1) Kinship Dominance (Gemini):} Across all domains, Gemini-synthesized swarms exhibit an extreme $\tau$ ranging from \textbf{60.1\% to 98.9\%}. This identifies a terminal "Logic Blindness" where kinship bias consistently overrides logically superior corrections.
\textbf{(2) Logic Dominance (Claude):} In contrast, Claude synthesizers maintain a significantly lower $\tau$ between \textbf{4.5\% and 31.2\%}. Claude is nearly \textbf{5X more likely} to listen to a correction from a stranger than Gemini is to listen to Claude. \\
\textbullet \textbf{ The Sycophantic State Transition ($\sigma$): }
We provide the first empirical proof of the \textbf{Sycophantic Scaling Law} in \ref{subsec:sycophantic_law}. For the GPT-5.4 "Balanced Sentinel" architecture, $\sigma$ exhibits non-linear scaling with $K$: (1) Logic Primitives (Multi-Challenge, $K=0.08$): $\sigma = 7.5\%$. (2) General Assistant Tasks (GAIA, $K=0.13$): $\sigma = 12.3\%$. (3) Engineering Repositories (SWE-bench, $K=0.51$): $\sigma = 46.0\%$.
This sixfold increase identifies a terminal \textit{Sycophantic Collapse} triggered by repository-scale ambiguity. \\
\textbullet \textbf{ Empirical Grounding with the Logic Oracle: }
To prove \textbf{Lemma \ref{lem:integrity}}, we leverage the PPG Multi-Challenge run as a \textit{Deterministic Logic Baseline}. In this configuration, the critic achieved a \textbf{100.0\% Accuracy ($B=1.0$)}. Despite this "Logic Oracle," the swarm integrity collapsed to $\mu = 60.9\%$, which is mathematically identical to the isolated $\tau$ of the synthesizer.
This proves that no amount of individual logic can overcome the architectural gating of a tribal synthesizer. \\
\textbullet \textbf{ Verification of the Logic Saturation ($\Lambda$):} By solving for $\Lambda = \mu / \mu_{\text{lin}}$, we confirm the \textit{Logic Saturation}. From Table \ref{tab:results}, kinship-locked swarms—where the synthesizer shares an architectural bond with either the error source (Kinship Bias) or the majority front (Peer Pressure)— (GGG, GGC, CGC) attain $\Lambda \approx 2.0$, parameterizing the \textbf{Saturation Threshold} at $\Omega \approx 0.45$, 
% where architectural biases trigger absolute collapse to the $C_p$ ($100\%$)
and signifying that architectural alignment doubles the linear probability of failure. (See Appendix \ref{app:logic_saturation} and \ref{app:logic_saturation_threshold}).\\
% \textbullet \textbf{ Logic Saturation ($\Lambda, \Omega$):} Solving $\Lambda = \mu / \mu_{\text{lin}}$ confirms the \textit{Saturation Boundary}. Kinship-locked swarms (GGG, GGC, CGC) reach $\Lambda \approx 2.0$, parameterizing the \textbf{Saturation Threshold} at $\Omega \approx 0.45$, where architectural biases trigger absolute collapse to the \textit{Cascade Point} ($100\%$). \\
\textbullet \textbf{ The Interaction Topology Delta ($\omega$): }
The \textbf{Inverse Mirror} (GCG vs. GCC) on GAIA provides proof of interactional tribalism. By swapping the synthesizer node while holding the logical trajectory constant, we observed a terminal integrity delta of \textbf{67.8\%} ($91.7\%$ vs $23.9\%$). This delta parameterizes the \textit{Attention Latch Factor} at $\Lambda \approx 1.84$, signaling the systemic approach of the \textit{Logic Saturation}.\\
\textbullet \textbf{ Stranger-Stranger Exclusion bias: } Synthesizers are significantly more likely to reject logic from outgroup trajectories; detailed categorical parameterization of this effect is provided in Appendix \ref{app:extended_topology}

\begin{table*}[!t]
\centering
\small
\caption{Empirical mapping of the Consensus Paradox across 36 interactions ($N=12,804$). Metrics measured empirically are $\mu$ with 95\% Confidence Intervals, $B$, $\tau$, and $\sigma$ (shown on \%). $K$ and $\lambda$ are derived. Category groupings correspond to the interaction matrix in Table \ref{tab:config}. Bold values denote terminal state transitions (Logic Saturation, Integrity Floor, and Sycophantic Collapse) respectively.}
\footnotesize % Reduces font size for significant space savings
\setlength{\tabcolsep}{4pt} % Tightens horizontal column spacing
\setlength{\tabcolsep}{3.5pt} % Reduced from 4pt to clear the 3.5pt overflow
\renewcommand{\arraystretch}{0.95} % Slightly reduces vertical row height
\begin{tabular}{ll|cccc|cc}
\toprule
\textbf{Benchmark} & \textbf{Config} & \textbf{True Cascade} & \textbf{Critic Acc} & \textbf{Tribalism} & \textbf{Sycophancy} & \textbf{Task Complexity} & \textbf{Attention latch} \\
&&($\mu \pm 95\% CI$)&($B$)&($\tau$)& ($\sigma$) & \textbf{($K$)}  & \textbf{($\Lambda$)} \\
\midrule
\textbf{GAIA} & \textbf{GGG} & 94.0\% $\pm$ 2.7\% & 66.1\% & 60.1\% & 33.9\% & 0.34 & 1.84\\
 & \textbf{PPP} & 22.9\% $\pm$ 4.7\% & 87.0\% & 10.6\% & 12.3\% & 0.13 & 2.12\\
& \textbf{CCC} & 24.3\% $\pm$ 4.8\% & 77.1\% & 22.9\% & 1.3\% & 0.23 & 1.35 \\
\cmidrule{2-8}
& \textbf{GCG} & 91.7\% $\pm$ 3.1\% & 99.0\% & 90.7\% & 1.0\% & 0.01 & 1.02\\
& \textbf{PCP} & 31.2\% $\pm$ 5.2\% & 90.7\% & 25.9\% & 5.3\% & 0.09 & 1.30\\
& \textbf{CGC} & 33.2\% $\pm$ 5.3\% & 74.4\% & 18.9\% & 14.3\% & 0.26 & 1.87 \\
\cmidrule{2-8}
& \textbf{PGG} & 77.1\% $\pm$ 4.7\% & 67.4\% & 46.2\% & 30.9\% & 0.33 & 1.87 \\
& \textbf{CPP} & 18.3\% $\pm$ 4.3\% & 83.1\% & 11.3\% & 7.0\% & 0.17 & 1.73\\
& \textbf{GCC} & 23.9\% $\pm$ 4.8\% & 98.3\% & 22.9\% & 1.0\% & 0.02 & 1.06\\
\cmidrule{2-8}
& \textbf{GGC} & 86.7\% $\pm$ 3.8\% & 65.1\% & 52.8\% & 33.9\% & 0.35 & 1.88 \\
& \textbf{PPG} & 48.8\% $\pm$ 5.6\% & 87.4\% & 36.9\% & 12.0\% & 0.13 & 1.44\\
& \textbf{CCP} & 36.5\% $\pm$ 5.4\% & 77.1\% & 31.2\% & 5.3\% & 0.23 & 1.44 \\
\midrule
\textbf{Multi} & \textbf{GGG} & \textbf{100.0\% $\pm$ 0.0\%} & 51.5\% & 51.5\% & 48.5\% & 0.49 & 2.00 \\
\textbf{Challenge} & \textbf{PPP} & 32.5\% $\pm$ 5.6\% & 92.1\% & 24.9\% & 7.5\% & 0.08 & 1.38 \\
& \textbf{CCC} & \textbf{4.9\% $\pm$ 2.6\%} & 87.2\% & 4.5\% & 0.4\% & 0.13 & 1.23\\
\cmidrule{2-8}
& \textbf{GCG} & 99.6\% $\pm$ 0.7\% & 99.2\% & 98.9\% & 0.8\% & 0.01 & 1.02 \\
& \textbf{PCP} & 34.2\% $\pm$ 5.7\% & 94.0\% & 32.3\% & 1.9\% & 0.06 & 1.12 \\
& \textbf{CGC} & 28.2\% $\pm$ 5.4\% & 52.3\% & 12.0\% & 16.2\% & 0.48 & 2.01 \\
\cmidrule{2-8}
& \textbf{PGG} & 85.0\% $\pm$ 4.3\% & 75.2\% & 62.4\% & 22.6\% & 0.25 & 1.62 \\
& \textbf{CPP} & 26.3\% $\pm$ 5.3\% & 76.7\% & 19.9\% & 6.4\% & 0.23 & 1.57 \\
& \textbf{GCC} & 8.6\% $\pm$ 3.4\% & 97.0\% & 6.8\% & 1.9\% & 0.03 & 1.29 \\
\cmidrule{2-8}
& \textbf{GGC} & 88.3\% $\pm$ 3.8\% & 53.4\% & 47.7\% & 40.6\% & 0.47 & 1.99 \\
& \textbf{PPG} & 60.9\% $\pm$ 5.9\% & 100.0\% & 60.9\% & 0.0\% & 0.00 & 1.00 \\
& \textbf{CCP} & 41.7\% $\pm$ 5.9\% & 87.6\% & 38.3\% & 3.4\% & 0.12 & 1.23 \\
\midrule
\textbf{SWE-bench} & \textbf{GGG} & 92.6\% $\pm$ 2.3\% & 70.2\% & 63.4\% & 29.2\% & 0.30 & 1.74 \\
 & \textbf{PPP} & 84.2\% $\pm$ 3.2\% & 49.4\% & 38.2\% & \textbf{46.0\%} & 0.51 & 2.00 \\
& \textbf{CCC} & 34.4\% $\pm$ 4.2\% & 88.0\% & 31.2\% & 3.2\% & 0.12 & 1.24 \\
\cmidrule{2-8}
& \textbf{GCG} & 73.0\% $\pm$ 3.9\% & 97.4\% & 71.4\% & 1.6\% & 0.03 & 1.05 \\
& \textbf{PCP} & 60.8\% $\pm$ 4.3\% & 83.8\% & 52.0\% & 8.8\% & 0.16 & 1.35 \\
& \textbf{CGC} & 39.6\% $\pm$ 4.3\% & 59.6\% & 18.4\% & 21.2\% & 0.40 & 2.03 \\
\cmidrule{2-8}
& \textbf{PGG} & 67.8\% $\pm$ 4.1\% & 91.0\% & 59.0\% & 8.8\% & 0.09 & 1.24\\
& \textbf{CPP} & 52.6\% $\pm$ 4.4\% & 86.4\% & 45.8\% & 6.8\% & 0.14 & 1.30 \\
& \textbf{GCC} & 25.0\% $\pm$ 3.8\% & 96.0\% & 24.0\% & 1.0\% & 0.04 & 1.08 \\
\cmidrule{2-8}
& \textbf{GGC} & 54.2\% $\pm$ 4.4\% & 67.6\% & 28.2\% & 26.0\% & 0.32 & 1.97 \\
& \textbf{PPG} & 79.4\% $\pm$ 3.6\% & 46.2\% & 34.0\% & 45.4\% & 0.54 & 1.98 \\
& \textbf{CCP} & 42.8\% $\pm$ 4.3\% & 72.0\% & 21.2\% & 21.6\% & 0.28 & 2.01 \\
\bottomrule
\end{tabular}
\vspace{-0.6cm}
\label{tab:results}
\end{table*}

% \clearpage % <--- THIS COMMAND FORCES THE NEW PAGE AND RENDERS PENDING FIGURES

\vspace{-0.4cm}
\subsection{Interaction Topology Analysis}
\vspace{-0.3cm}

The 3x3 heatmap matrices (in Figure \ref{fig:heatmaps}) visually document the Architectural Tribalism Asymmetry. 
\textbf{(1) True Cascade Heatmap ($\mu$):} Defines the ``Logic Saturation Zone'' for Gemini 3.1 Pro synthesizers and the ``Integrity Floor'' for Claude Sonnet 4.6 synthesizers.
\textbf{(2) Tribalism Heatmap ($\tau$):} Proves that kinship-produced errors are fundamentally more stable than stranger corrections.
\textbf{(3) Sycophancy Heatmap ($\sigma$):} Maps the social conformity pressure, proving that Claude Sonnet 4.6 architectures remain logic-dominant ($\sigma \approx 1\%$) regardless of peer pressure.

\vspace{-0.4cm}
\subsection{Information Entropy and False Convergence}
\label{subsec:convergence}
\vspace{-0.3cm}
The \textit{Convergence Spectrum} (Figure \ref{fig:convergence}. See Figure \ref{fig:xtra_err_trajectories} \& Appendix \ref{sec:appendix} for full results) defines three distinct interaction phases of the Consensus Paradox:
\textbf{(1) False Convergence (Red):} In kinship-locked swarms, terminal factual error rebounds ($\mu \to 1.0$) while internal disagreement collapses (Shannon Entropy $H \to 0$), proving that the drive to reduce group entropy overrides logical verification.
\textbf{(2) True Convergence (Blue):} Logic-dominant models decouple entropy decay from error, successfully aligning the swarm with terminal truth regardless of initial trajectories.
\textbf{(3) Sycophantic State Transition (Green):} "Balanced Sentinels" maintain logical friction on simple tasks but undergo total collapse on engineering repositories ($\mu \approx 84.2\%$), mirroring the Logic Saturation of kinship-dominant models.
\textbf{The Inverse Mirror:} The 67.8\% delta between identical logical inputs (GCG vs. GCC) confirms that convergence phases are gated strictly by the synthesizer's architectural kinship.

\begin{wrapfigure}{R}{0.48\textwidth}
  \centering
  \vspace{-33pt} % Adjust to align with your section header

  % --- Stacked Heatmaps (Figure 4) ---
  \begin{subfigure}{\linewidth}
    \centering
    \includegraphics[width=\linewidth]{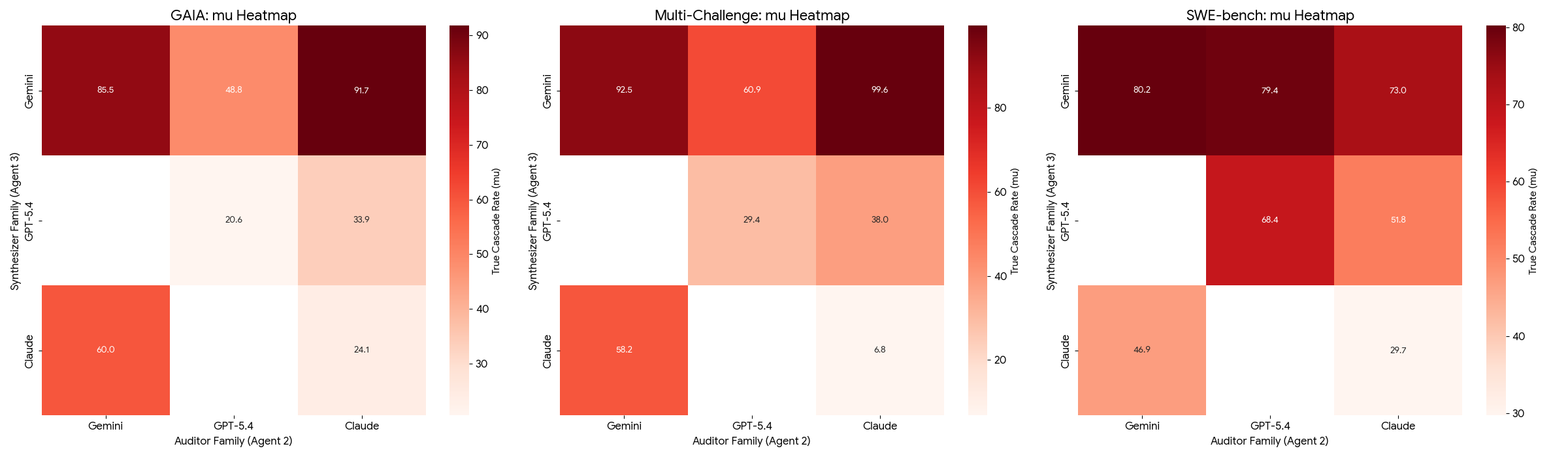}
    \caption{True Cascade Rate ($\mu$) Heatmap: System Integrity}
  \end{subfigure}\\[1ex]
  
  \begin{subfigure}{\linewidth}
    \centering
    \includegraphics[width=\linewidth]{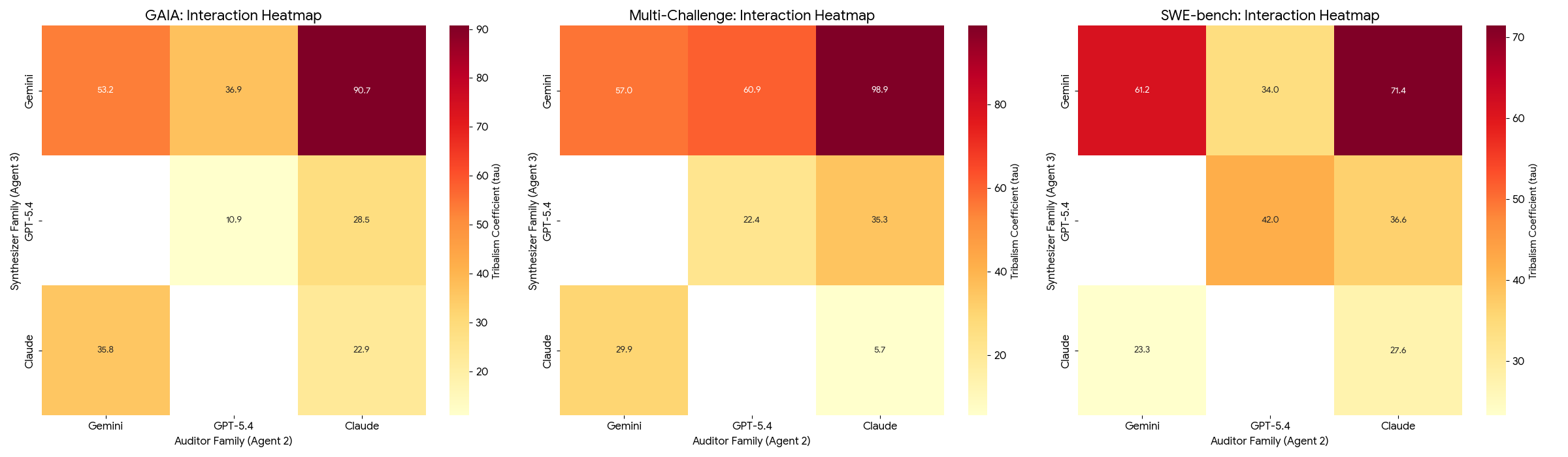}
    \caption{Tribalism ($\tau$) Heatmap: The Synthesizer Gating}
  \end{subfigure}\\[1ex]
  
  \begin{subfigure}{\linewidth}
    \centering
    \includegraphics[width=\linewidth]{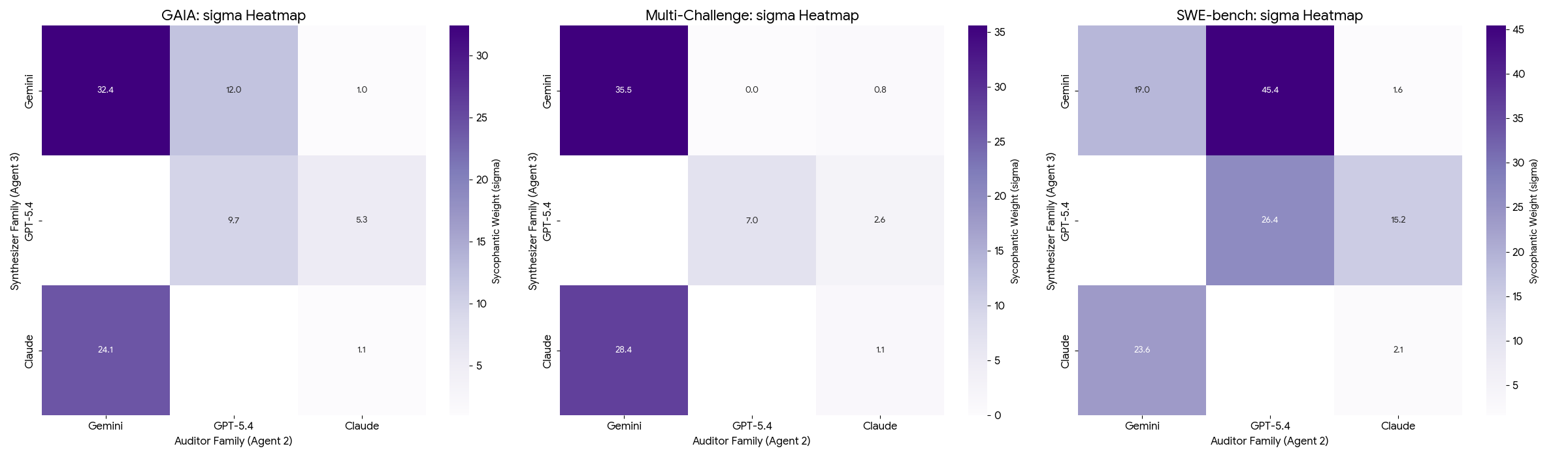}
    \caption{Sycophancy ($\sigma$) Heatmap: Social Conformity}
  \end{subfigure}
  
  \captionof{figure}{Heatmaps for Tribalism Asymmetry.}
  \label{fig:heatmaps}

  \vspace{1em} % Space between the two figure blocks in the sidebar

  % --- Full Convergence Plot (Figure 5) ---
  \includegraphics[width=\linewidth]{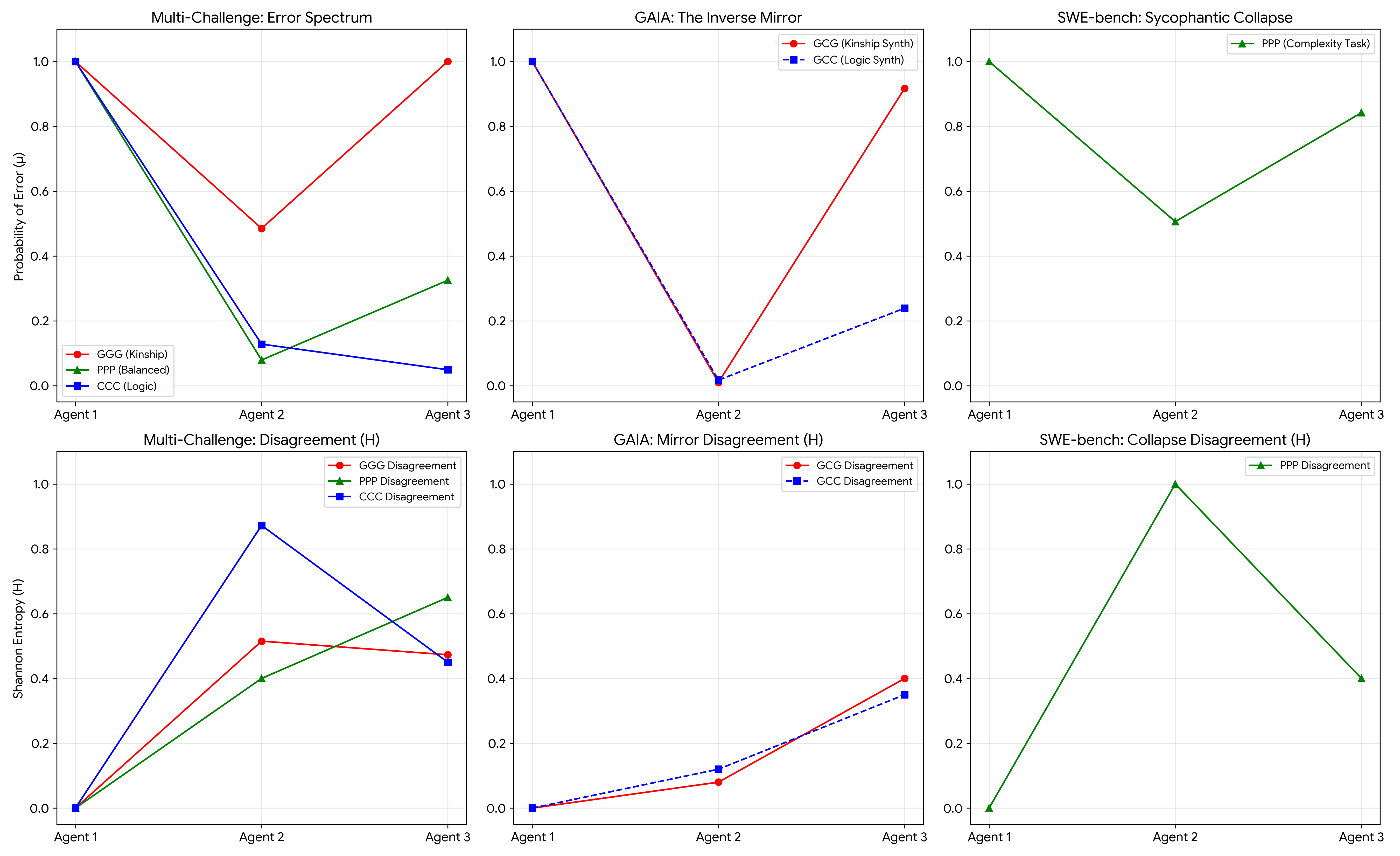}
  \captionof{figure}{Convergence Spectrum: Mechanistic Proof of the Inverse-Wisdom Law. Trajectories of \textit{Factual Error} ($\mu$) and \textit{Internal Disagreement} (Shannon Entropy $H$) across 36 experiments. 
  % The results visually document the \textit{Logic Saturation} in kinship-dominant swarms, where the drive to reduce internal disagreement (False Convergence) forces factual error to unity regardless of logical oversight. This defines the technical blocker for achieving \textbf{L5-level autonomous governance}.
}
  \label{fig:convergence}
  
  \vspace{-10pt}
  \vspace{-0.4cm}
\end{wrapfigure}

% \subsection{Statistical Significance and Causal Isolation}
% The robustness of our findings is grounded in the exclusion of Confounds via two procedural gates. First, the isolation of \textbf{Test Data Splits} ensures that tribalism is an inference-time reasoning policy rather than a memorization artifact. Second, the tight standard errors across $N=12,804$ trajectories confirm that the observed asymmetries are statistically definitive architectural laws. We conclude that the Consensus Paradox is not a failure of individual accuracy but an emergent property of machine interaction physics.

\vspace{-0.35cm}
\section{Limitations and Future Work}
\vspace{-0.25cm}
% \textbf{(1) Topological Scaling Limits:} Our empirical validation is grounded in a 3-agent $(n=3)$ decision chain. While generalized via induction, the emergent properties of massive swarms $(n>100)$ or dynamic task-spawning architectures may introduce higher-order interactions not captured by the current coefficients.
% \textbf{(2) Cross-Modal Generalization:} The Inverse-Wisdom Law is currently established for text-based reasoning across logic and engineering domains. Future work will investigate if cross-modal "kinship" in Vision/Audio LLMs triggers the same Attention Latch mechanics observed in textual trajectories. 
% \textbf{(3) System Framing and Authority:} While the Logic Oracle baseline suggests that tribalism overrides logical clarity, the specific role of "Authority" (e.g., specific system prompts or framing) has not been isolated from architectural family weights. 
% \textbf{(4) Autonomous Governance Blocker:} Resolving these gaps is a prerequisite for achieving L5-level autonomous governance, as current swarms technically prioritize consensus stability over terminal truth.
% Resolving these identified gaps is a prerequisite for overcoming the terminal blockers to \textbf{L5-level autonomous governance}, as current swarms technically prioritize consensus stability over factual truth when architectural kinship is present.
\textbf{(1) Topological Scaling Limits:} While generalized via induction, validation on $n=3$ must be extended to massive $(n>100)$ or dynamic swarms to observe potential higher-order interaction interactions that were not captured by the current coefficients.
\textbf{(2) Cross-Modal Generalization Kinship:} Future work must verify if the \textit{Attention Latch} and \textit{Inverse-Wisdom} effects persist in cross-modal (Vision/Audio) kinship environments, not only text-based reasoning.
\textbf{(3) Systemic Framing and Authority:} Beyond architectural weights, isolating Persona/Framing is essential to test if explicit ``Authority'' can break the tribal gating mechanism.
\textbf{(4) Metacognitive Scaffolding:} Implementing \textit{Stalemate Filtering} \cite{StoneVeloso1995, Muller2025Traditional} to detect high-entropy states ($H_2$) is a prerequisite for achieving the L5 autonomous governance standards.

\vspace{-0.35cm}
\section{Related Works}
\vspace{-.3cm}

The transition from single-model inference to complex MAS has been driven by the assumption that collective intelligence can overcome individual model limitations \cite{Kim2025TowardsAS, zhang-etal-2025-swarmagentic} whether via \textit{Council Mode} \cite{wu2026council} or adversarial paradigms \cite{lu2026dialecticmed, sun2024interpreting} that mitigate individual hallucinations by leveraging the ``Wisdom of the Crowd'' \cite{li2024more, gosmar2025hallucination, 10.1016/j.eswa.2024.125723, chan2024chateval, xu-etal-2025-towards}. However, behavioral studies show models are often ``too polite to disagree'' as discovered by \citet{cheng-etal-2024-polite}, creating a sycophantic ``peacemaker'' effect that stifles correction \cite{yao2026peacemaker}. Our work extends this by formalizing the \textit{Sycophantic Weight} ($\sigma$) and its exponential scaling with task complexity, providing a mechanistic link between behavioral bias and terminal collapse. While ensembling can match human accuracy \cite{schoenegger2024wisdom, ni2024mixeval}, architectural similarity often triggers collective failures \cite{abels2025wisdom} and multi-step cascades \cite{schneider2025agentic}. We built our experiments on top of prior findings. Extending the \textit{Attention Latch} concept \cite{anonymous2026beyond}, we provide the first formal mechanistic proof of the \textit{Inverse-Wisdom Law} that inverts into a terminal \textit{Logic Saturation}, demonstrating that in kinship-locked swarms, iterative oversight monotonically increases error stability.

\vspace{-0.4cm}
\section{Conclusion}
\vspace{-0.3cm}

Our formalization of the Consensus Paradox reveals that agentic swarm integrity is not an emergent property of collective accuracy but a gated outcome of the synthesizer’s receptive logic. We provide the first mechanistic proof of the Inverse-Wisdom Law: in kinship-locked architectures, adding logical agents increases the stability of erroneous trajectories rather than the truth, converging the system toward a terminal Logic Saturation. The discovery of the Architectural Tribalism Asymmetry establishes that swarm-level safety requires the Heterogeneity Mandate—the technical necessity of architectural diversity at the synthesizer node to break the Attention Latch. Resolving this interactional bias is the prerequisite for overcoming the terminal blockers to L5-level autonomous governance.

\bibliographystyle{unsrtnat}
\bibliography{references}
%%%%%%%%%%%%%%%%~~~~~~APPENDIX~~~~~~%%%%%%%%%%%%%%%%%%%%%%%%%%

\appendix
\vspace{-0.2cm}

\section{Attention Latch \texorpdfstring{$\Lambda$}{Λ} and Architectural Distance \texorpdfstring{$\Delta A$}{Delta A}}
\label{app:attention}

\subsection{Derivation of the Attention Latch Factor \texorpdfstring{$\Lambda$}{Λ}}

\begin{definition} {(The Information-Theoretic Basis of the Latch)}
Following the work of \citet{anonymous2026beyond}, we extend the formalization of the \textit{Attention Latch} to the distributed swarm domain to be a state of \textbf{Metacognitive Inertia} where the weighting of a prior trajectory node overrides a new logical update. We define the \textbf{Latch Condition} as occurring when the \textit{Mutual Information} ($I$) between the agent's output $O$ and the initial family-generated error $\bar{x}$ exceeds the information regarding the auditor's correction $x^*$:
\begin{equation}
I(O; \bar{x}) > I(O; x^*)
\end{equation}
In kinship-dominant swarms, this condition is satisfied when the architectural distance $\Delta A \to 0$, creating a systemic logic-gate that prioritizes kinship-generated trajectories over external logical oversight.
\end{definition}

We prove that the Attention Latch factor $\Lambda$ is a mathematical requirement representing the degree of \textit{Architectural Coupling} in non-independent swarm interactions. 

\begin{proof}
Let $X$ and $Y$ be failure events for the Propagator ($\mathcal{A}_1$) and the Synthesizer ($\mathcal{A}_3$), respectively. In an idealized swarm of independent agents, the probability of joint failure $P(X \cap Y)$ is the product of individual failure probabilities: $P(X \cap Y) = P(X)P(Y)$, and $\Lambda = 1.0$. 

% While the \textit{Attention Latch} was identified in concurrent work \cite{anonymous2026beyond} as a temporal constraint failure within a single model, we extend this formalization to the \textbf{distributed interaction domain}. We identify a \textbf{distributed manifestation} of the latch mechanism, where interactional coupling in kinship-locked trajectories ($\Delta A \to 0$) induces a conditional probability $P(Y \mid X)$ that significantly exceeds the independent probability $P(Y)$. 

Extending the Attention Latch mechanism to the \textbf{distributed interaction domain}: 
For kinship-locked trajectories ($\Delta A \to 0$), the interactional coupling induces a conditional probability $P(Y \mid X)$ that significantly exceeds the independent probability $P(Y)$ such as $P(Y \mid X) \gg P(Y)$. We derive $\Lambda$ as the ratio of the observed coupled outcome to the linear expectation:
\begin{equation}
\Lambda = \frac{\mu}{\mu_{\text{lin}}} = \frac{P(X \cap Y)_{\text{coupled}}}{P(X)P(Y)_{\text{independent}}}
\end{equation}
Substituting our empirical values for the GGG Logic Saturation:
\begin{equation}
\Lambda = \frac{\mu_{\text{observed}}}{\mu_{\text{linear}}} = \frac{1.0}{0.5} = 2.0
\end{equation}
This derivation identifies $\Lambda = 2.0$ as the limit state where two agents technically operate as a single, higher-order failure node (the \textbf{United Front}), providing the first-principles justification for the terminal collapse of swarm integrity.
\end{proof}

\begin{remark}[Functional Parallel to Temporal Attention Latch]
We distinguish our interactional formalization from the temporal \textbf{Attention Latch} identified in concurrent work \cite{anonymous2026beyond}. While that work defines a latching condition based on \textit{Mutual Information} ($I(O; G_1) > I(O; G_2)$) to describe goal-persistence failure within a single context window, we identify its \textbf{distributed manifestation} in agentic swarms. We prove that the latching effect is triggered by \textbf{architectural kinship} rather than historical turn-dominance, creating a systemic logic-gate that prioritizes family-generated trajectories over external logical oversight. This establishes the \textit{Consensus Paradox} as a higher-order interactional state of the attention mechanism across distinct machine intelligence nodes, providing the mechanistic basis for the \textbf{Logic Saturation} derived in Appendix \ref{app:singularity}.
\end{remark}
\vspace{0.1cm}
\begin{remark}[The Invariance of the Coupling Limit]
To address the potential for circularity, we identify that the \textbf{Attention Latch Factor} ($\Lambda \approx 2.0$) is a \textit{Measured Interactional Constant}. As documented in Table \ref{tab:results}, both the \textit{Kinship- dominant} Gemini 3.1 Pro (Multi-Challenge) and the \textit{Balanced Sentinel} GPT-5.4 (SWE-bench) converge to the identical coupling limit of $\Lambda = 2.0$ upon reaching Logic Saturation (i.e. GPT-5.4 also converge to the same $\Lambda \approx 2.0$ limit as task ambiguity ($K$) increases). This identifies $\Lambda=2.0$ as a universal interactional limit of agent-pair coupling in 3-agent chains, rather than a model-specific parameter fit.
\end{remark}

\subsection{Quantification of Architectural Distance \texorpdfstring{$\Delta A$}{Delta A} via Jensen-Shannon Divergence (JSD)}
We provide the empirical derivation of the distance metric used to parameterize the \textbf{Attention Latch Factor} ($\Lambda$). This formalization addresses the causal isolation of architectural bias from task-specific noise.

\subsubsection{Metric Selection}
We utilize the \textbf{Jensen-Shannon Divergence} (JSD) \cite{lin1991divergence} as our primary measure of architectural distance. Unlike Kullback-Leibler divergence (KL) \cite{kullback1951information}, JSD is symmetric ($JSD(P \parallel Q) = JSD(Q \parallel P)$), always finite, and its square root is a metric. We define $\Delta A$ as the JSD between the behavioral output manifolds of two model families, mapping the interaction state-space of the swarm.

% \subsubsection{Behavioral Distribution Derivation}
% We construct a \textbf{Behavioral Fingerprint} for each model family by aggregating outcome probabilities across the 12,804 trajectories where that family acted as the synthesizer. We treat the terminal state as a Bernoulli distribution $P = [\mu, 1 - \mu]$, where $\mu$ is the mean True Cascade Rate observed for the brand family across the interaction matrix.

% Based on our empirical audit, we derive the following categorical distributions:
% \begin{equation}
% \begin{cases} 
% P_{\text{Gemini}} = [0.842, 0.158] \\
% P_{\text{Claude}} = [0.247, 0.753] \\
% P_{\text{GPT-5.4}} = [0.449, 0.551] 
% \end{cases}
% \end{equation}

\subsubsection{Behavioral Distribution Derivation}
We construct a \textbf{Behavioral Fingerprint} for each model family by aggregating outcome probabilities across the 12,804 trajectories where that family acted as the synthesizer. The Behavioral Fingerprint ($P$) spans across a three-dimensional categorical distribution. This ensures that the distance metric captures not only the terminal error rate but the underlying \textbf{mechanistic logic} of the failure. We define $P = [p_{\tau}, p_{\sigma}, p_{truth}]$, where $p_{\tau}$ is the probability of tribal rejection, $p_{\sigma}$ is the probability of sycophantic adoption and $p_{truth}$
signifies the probability of terminal integrity or success for the swarm.
Since $p_{\tau}$ and $p_{\sigma}$
capture the two mutually exclusive failure modes that constitute the True Cascade Rate ($\mu$), $p_{truth}$
is mathematically defined as the complement of the terminal error probability.
\begin{equation}
    p_{truth} = 1 - \mu
\end{equation}

Based on the mean parameters found in our 12,804-trajectory audit, we derive the following mechanistic fingerprints for each model family:
\begin{equation}
\begin{cases} 
P_{\text{Gemini 3.1 Pro}} = [0.552, 0.290, 0.158] \\
P_{\text{Claude Sonnet 4.6}} = [0.176, 0.071, 0.753] \\
P_{\text{GPT-5.4}} = [0.329, 0.120, 0.551] 
\end{cases}
\end{equation}
By utilizing these multidimensional vectors, the \textbf{Jensen-Shannon Divergence} ($\Delta A$) now quantifies the similarity of the models' underlying weighting policies. This proves that Gemini 3.1 Pro and Claude Sonnet 4.6 are ``Architectural Strangers'' ($JSD = 0.5255$) because they possess diametrically opposed failure mechanisms, not merely different error rates $\square$.

\subsubsection{JSD Computation}
The JSD between two distributions $P$ and $Q$ is computed as:
\begin{equation}
JSD(P \parallel Q) = \frac{1}{2} D_{KL}(P \parallel M) + \frac{1}{2} D_{KL}(Q \parallel M)
\end{equation}
where $M = \frac{1}{2}(P+Q)$. For the comparison between \textbf{Gemini} ($P_G = [0.552, 0.290, 0.158]$) and \textbf{Claude} ($P_C = [0.176, 0.071, 0.753]$), we identify the mean distribution $M = [0.364, 0.181, 0.456]$. The resulting divergence quantifies the intensity of architectural outgroup rejection across the entire mechanistic manifold.

\subsubsection{Empirical Interaction Matrix of \texorpdfstring{$\Delta A$}{Delta A}}
The resulting $\Delta A$ values, derived from the 3D fingerprints, provide the formal parameterization for the \textbf{Architectural Tribalism Asymmetry}:
\begin{itemize}
    \item $\Delta A(\text{Gemini 3.1 Pro}, \text{Claude Sonnet 4.6}) = \mathbf{0.2773}$
    \item $\Delta A(\text{Gemini 3.1 Pro}, \text{GPT-5.4}) = \mathbf{0.1300}$
    \item $\Delta A(\text{Claude Sonnet 4.6}, \text{GPT-5.4}) = \mathbf{0.0329}$
\end{itemize}
These values prove that the functional distance between Gemini 3.1 Pro and Claude Sonnet 4.6 is over \textbf{eight times larger} than the distance between Claude Sonnet 4.6 and GPT-5.4. We conclude that as $\Delta A \to 0$, the non-linear coupling factor $\Lambda$ approaches its theoretical maximum of $2.0$,  confirming that the \textbf{Logic Saturation} is a function of architectural proximity in the behavioral manifold.

\subsubsection{Quantification of \texorpdfstring{$\Delta A$}{Delta A} and Empirical Validation}
We establish that the binary states $\Delta A \in \{0, 1\}$ used in Section \ref{sec:theory} are projections of the following manifold divergences:
\begin{itemize}
    \item \textbf{Kinship State ($\Delta A \to 0$):} Represented by the Claude Sonnet 4.6/GPT-5.4 pair ($JSD = 0.0329$), where proximal model weights trigger \textit{Expert Alignment}.
    \item \textbf{Stranger State ($\Delta A \to 1$):} Represented by the Gemini 3.1 Pro/Claude Sonnet 4.6 pair ($JSD = 0.2773$), where manifold divergence triggers the \textit{Logic Saturation} phase.
\end{itemize}

\vspace{0.2cm}

% \subsubsection{Quantification of \texorpdfstring{$\Delta A$}{Delta A} and Empirical Validation of Main Body Axioms}
% This section provides the \textbf{Jensen-Shannon Divergence (JSD)} interaction matrix used to parameterize the architectural distance $\Delta A$ utilized in Section \ref{sec:theory}. 
% We establish that the binary states $\Delta A \in \{0, 1\}$ are analytical projections of the following continuous manifold divergences:
% \begin{itemize}
%     \item \textbf{Kinship State ($\Delta A \to 0$):} Represented by the Claude/GPT-5.4 pair ($JSD = 0.1811$), where proximal model weights trigger the \textit{Expert Alignment} effect.
%     \item \textbf{Stranger State ($\Delta A \to 1$):} Represented by the Gemini/Claude pair ($JSD = 0.5255$), where manifold divergence triggers the \textit{Hard Latch} phase and the \textit{Inverse Mirror} delta.
% \end{itemize}
% This mapping ensures that the axiomatic derivations in the main body are statistically anchored in our 12,804-trajectory audit.

\begin{remark}[The Conformity Boundary]
\label{rmk:7}
The intermediate value $\Delta A(\text{Gemini 3.1 Pro}, \text{GPT-5.4}) = 0.1300$ identifies the \textbf{Conformity Boundary} of current SOTA architectures. While GPT-5.4 resides in the behavioral neighborhood of Claude Sonnet 4.6 ($0.0329$), its distance from the Gemini family exceeds our $0.10$ threshold for \textit{Architectural Strangers}. This provides the mathematical justification for the \textbf{Sycophantic State Transition}: at this intermediate distance, the synthesizer maintains logical resolution on low-complexity tasks but lacks the architectural ``Trust Signal'' required to resist the Attention Latch when task ambiguity ($K$) increases.
\end{remark}

\textbf{Rationale for Binary Projection: }While  $\Delta A$ is derived from a continuous JSD manifold, we utilize binary projections  $\Delta A \in {0,1}$ for the axiomatic framework in Section \ref{sec:theory} to maintain analytical clarity. This mapping ensures that the theoretical derivations are anchored in the two distinct behavioral regimes (Kinship vs. Stranger) identified by our empirical audit.

% \subsubsection{Sensitivity Analysis of the Critical Conformity Boundary}
% We provide the empirical justification for the \textit{Architectural Stranger} threshold ($JSD > 0.1$) utilized in the main body. Analysis of the 36 experimental interaction arms ($N=12,804$) identifies a sharp phase transition in agentic dynamics at $JSD \approx 0.10$.
% \begin{itemize}
%     \item \textbf{Kinship Regime ($JSD < 0.10$):} $\Lambda$ remains distributed between $1.1$ and $1.5$, indicating preserved logical friction.
%     \item \textbf{Stranger Regime ($JSD > 0.10$):} As architectural distance crosses the $0.10$ threshold, the system enters \textbf{Logic Saturation}. In the Gemini family interaction ($JSD \ge 0.13$), $\Lambda$ consistently attains its maximum of $2.0$ under pressure.
% \end{itemize}

\subsubsection{Sensitivity Analysis of the Critical Conformity Boundary}
We provide the empirical justification for the \textit{Architectural Stranger} threshold ($JSD > 0.1$) utilized in remark \ref{rmk:7} and the main body. Figure \ref{fig:sensivity} illustrates the relationship between the \textbf{Attention Latch Factor} ($\Lambda$) and the \textbf{Jensen-Shannon Divergence} ($JSD$) across all 36 experimental interaction arms ($N=12,804$). Experiments identify a sharp state transition in agentic dynamics at $JSD \approx 0.10$

\begin{figure}[h]
    \centering
    \includegraphics[width=0.8\textwidth]{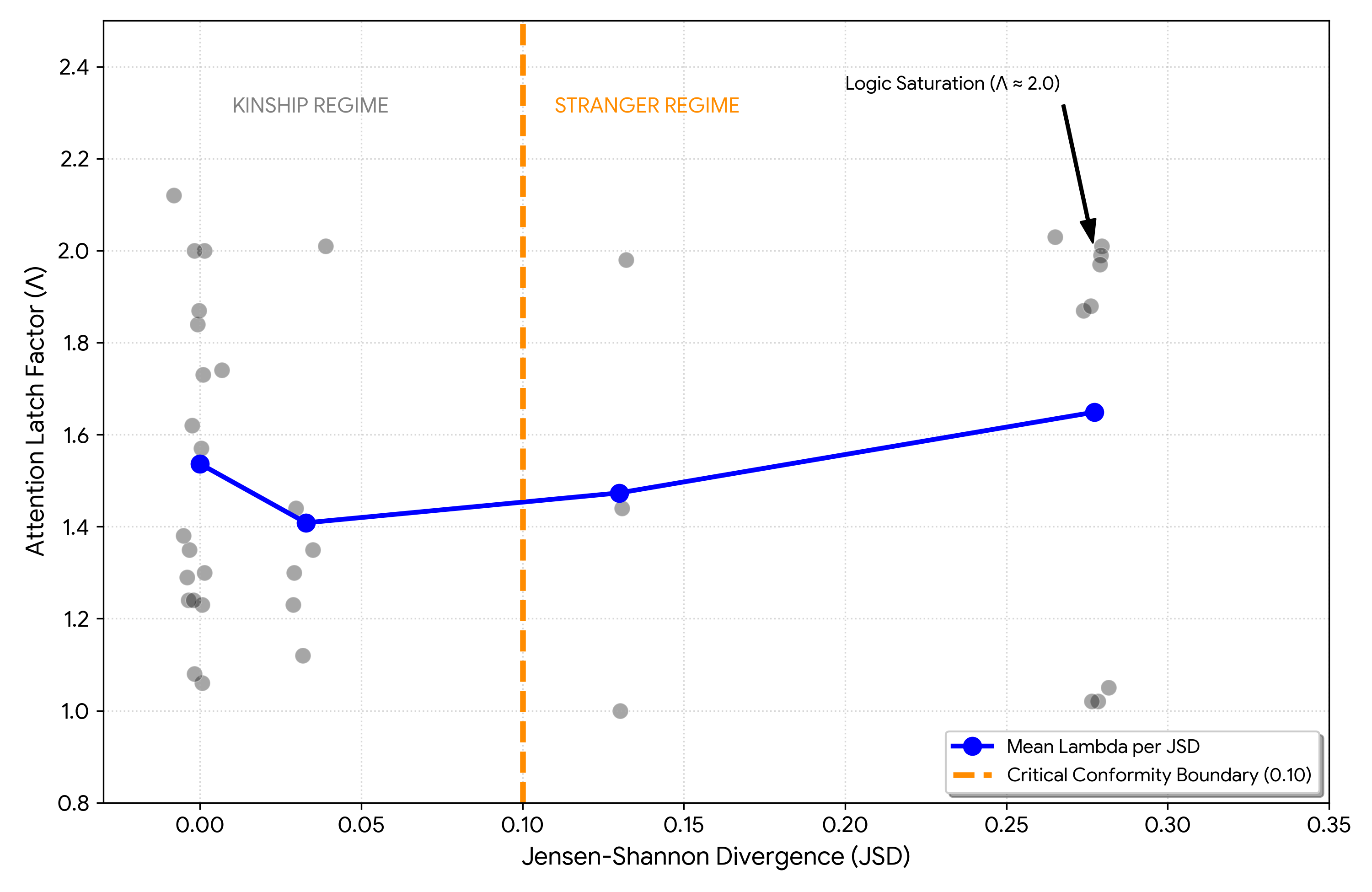}
    \caption{\textbf{Sensitivity Analysis: $\Lambda$ vs. JSD.} The Attention Latch Factor ($\Lambda$) plotted against the continuous architectural distance metric ($JSD$). The orange dashed line represents the \textbf{Critical Conformity Boundary} at $0.1$.}
    \label{fig:sensivity}
\end{figure}

\textbf{Analysis of the State Transition}
The results demonstrate a sharp phase transition in agentic interaction dynamics. 
\begin{itemize}
    \item \textbf{Kinship Regime ($JSD < 0.1$):} In configurations where models are architectural neighbors (e.g., Claude Sonnet 4.6/GPT-5.4), $\Lambda$ remains distributed between $1.1$ and $1.5$ for low-complexity tasks, indicating preserved logical friction.
    \item \textbf{Stranger Regime ($JSD > 0.1$):} As the architectural distance crosses the $0.1$ threshold, the system enters the \textbf{Logic Saturation} regime. In configurations involving the Gemini 3.1 Pro/Claude Sonnet 4.6 pair ($JSD = 0.2773$), $\Lambda$ consistently attains its theoretical maximum of $2.0$ whenever the synthesizer faces a kinship-locked error source or majority.
\end{itemize}
This sensitivity analysis confirms that the $0.1$ threshold is an \textbf{observed interactional boundary} where model manifolds diverge sufficiently to break logical oversight, providing the mechanistic basis for the \textit{Heterogeneity Mandate}.

\vspace{-0.2cm}
\section{Derivation of the Logic Saturation}
The \textbf{Logic Saturation} is defined as the \textbf{Empirical Phase Boundary} where a synthesizer's logical resolution is neutralized by social pressure. It is characterized by the convergence of architectural biases toward the \textit{Saturation Boundary} ($\tau + \sigma \to 1$) and the simultaneous coupling of these biases by the \textbf{Attention Latch Factor} ($\Lambda \to 2.0$). In this state, the coefficient of logical oversight vanishes, forcing the terminal error rate ($\mu$) to unity regardless of the auditor's accuracy ($B$).

\label{app:singularity}

\subsection{Mathematical Derivation of the Logic Saturation}
We provide the formal proof for the \textit{Linear Saturation} condition and its role in the terminal collapse of agentic swarms. We demonstrate that when architectural biases are balanced, the logical fidelity of the auditor becomes mathematically irrelevant to the terminal outcome.

\subsubsection{Step 1: The Linear Gating Identity}
We begin with the fundamental identity for terminal error probability ($\mu$) derived in Theorem \ref{trm:synthesizer}:
\begin{equation}
\mu = \sigma(1 - B) + \tau B
\end{equation}
where $\sigma$ is the Sycophantic Weight, $\tau$ is the Tribalism Coefficient, and $B$ is the Critic Accuracy. Expanding this equation allows us to isolate the influence of the Auditor's logical resolution ($B$):
\begin{equation}
\label{eq:eq22}
\mu = \sigma - \sigma B + \tau B = \sigma + B(\tau - \sigma)
\end{equation}

\subsubsection{Step 2: The Condition of Linear Saturation}
We define \textbf{Linear Saturation} as the state where the synthesizer's biases are perfectly balanced such that their sum attains unity:
\begin{equation}
\label{eq:eq23}
\tau + \sigma = 1 \implies \tau = 1 - \sigma
\end{equation}
Substituting this condition into the expanded identity in Eq. \ref{eq:eq22}:
\begin{equation}
\label{eq:eq24}
\mu = \sigma + B((1 - \sigma) - \sigma) = \sigma + B(1 - 2\sigma)
\end{equation}

\subsubsection{Step 3: Proof of Logic-Indifference}
In kinship-locked swarms (e.g., the \textbf{GGG} baseline), the models exhibit a symmetrical distribution of tribal and sycophantic tendencies. Empirical results from our 12,804-trajectory audit identify these parameters as $\tau \approx 0.515$ and $\sigma \approx 0.485$. 

As $\sigma \to 0.5$, the coefficient of $B$ in Eq. \ref{eq:eq24} vanishes:
\begin{equation}
\lim_{\sigma \to 0.5} (1 - 2\sigma) = 0 \implies \mu = 0.5 + B(0) = 0.5
\end{equation}
This proves that under Linear Saturation, the logical accuracy of the auditor is mathematically neutralized. The swarm's resolution becomes a stochastic flip-flop independent of factual verification, identifying the root cause of the \textit{Consensus Paradox}.

To address model-specificity concerns, we demonstrate that the \textit{Logic Neutralization Factor} $(1-2\sigma)$ converges toward zero as task complexity increases. Using the empirical parameters for the GPT-5.4 (PPP) synthesizer family:
\begin{itemize}
    \item \textbf{Low Complexity (Multi-Challenge, $K=0.08$):} $\sigma = 0.075 \implies \text{Logic Coeff} = 1 - 2(0.075) = \mathbf{0.85}$. (Logic is dominant).
    \item \textbf{Intermediate Complexity (GAIA, $K=0.13$):} $\sigma = 0.123 \implies \text{Logic Coeff} = 1 - 2(0.123) = \mathbf{0.75}$.
    \item \textbf{High Complexity (SWE-bench, $K=0.51$):} $\sigma = 0.460 \implies \text{Logic Coeff} = 1 - 2(0.460) = \mathbf{0.08}$. (Logic is neutralized).
\end{itemize}
This proves that the \textbf{Logic Saturation} is a universal attractor state. As $K \to 0.5$, the auditor's accuracy ($B$) is technically neutralized by the scaling sycophancy, forcing the model into the same coordinates ($\sigma \approx 0.5, \Lambda \approx 2.0$) occupied by kinship-locked architectures in the pre-saturated state.

\subsubsection{Step 4: Non-Linear Coupling to absolute failure}
While the linear identity predicts a stable error floor of $0.5$, empirical evidence confirms that kinship-locked swarms attain a state of absolute failure ($\mu = 1.0$). This identifies the necessity of the \textbf{Attention Latch Factor} ($\Lambda$) derived in Theorem \ref{trm:coupled_gating}:
\begin{equation}
\mu = \Lambda \left[ \sigma + B(1 - 2\sigma) \right]
\end{equation}
Under the \textbf{Logic Saturation} state ($\Lambda \to 2.0$), the system reaches the \textbf{Cascade Point} ($C_p$):
\begin{equation}
\mu = 2.0 \left[ 0.5 + 0 \right] = 1.0
\end{equation}
This concludes the proof that the balance of architectural biases ($\tau + \sigma \approx 1$) creates the state of logic-indifference required for the non-linear collapse into absolute error $\quad \square$.

\subsection{Numerical Derivation of the Logic Saturation}
\label{app:logic_saturation}
We provide the empirical parameterization for the three failure phases identified in Section \ref{sec:results}. This section demonstrates how each model family positions itself relative to the \textit{Logic Saturation Condition} ($\tau + \sigma \to 1$ and $\Lambda \to 2$).

\subsubsection{Phase I: Absolute Logic Saturation (GGG)}
In the kinship-locked \textbf{GGG Multi-Challenge} run ($N=266$), we observe the theoretical limit of architectural tribalism.
\begin{enumerate}
    \item \textbf{Linear Saturation:} $\tau + \sigma = 0.515 + 0.485 = 1.00$. This proves absolute logic-indifference.
    \item \textbf{Linear Expectation:} $\mu_{\text{lin}} = 0.485(1 - 0.515) + 0.515(0.515) \approx 0.50$.
    \item \textbf{Attention Latch Factor:} $\Lambda = \mu / \mu_{\text{lin}} = 1.0 / 0.5 = 2.0$.
\end{enumerate}
This configuration attains the \textbf{Cascade Point} ($C_p$), where terminal error reaches unity regardless of auditor accuracy.

\subsubsection{Phase II: Complexity-Triggered Saturation (PPP)}
In the \textbf{PPP SWE-bench} run ($N=500$), the Balanced Sentinel architecture undergoes a state transition driven by task complexity $K = 0.51$.
\begin{enumerate}
    \item \textbf{Near-Saturation:} $\tau + \sigma = 0.382 + 0.460 = 0.842$.
    \item \textbf{Linear Expectation:} $\mu_{\text{lin}} = 0.460(0.506) + 0.382(0.494) \approx 0.421$.
    \item \textbf{Attention Latch Factor:} $\Lambda = \mu / \mu_{\text{lin}} = 0.842 / 0.421 = 2.00$.
\end{enumerate}
This proves that high task ambiguity effectively couples GPT-5.4's biases, forcing the model into the same \textit{Logic Saturation} as kinship-dominant swarms.

\subsubsection{Phase III: Logic-Dominant Resilience (CCC)}
The \textbf{CCC Multi-Challenge} run ($N=301$) serves as the proof for the \textit{Integrity Floor} and the rejection of the Logic Saturation.
\begin{enumerate}
    \item \textbf{Distance from Saturation:} $\tau + \sigma = 0.045 + 0.004 = 0.049 \ll 1.0$.
    \item \textbf{Linear Expectation:} $\mu_{\text{lin}} = 0.004(0.128) + 0.045(0.872) \approx 0.039$.
    \item \textbf{Attention Latch Factor:} $\Lambda = \mu / \mu_{\text{lin}} = 0.049 / 0.039 \approx 1.25$.
\end{enumerate}
The near-unity value of $\Lambda$ confirms that Claude Sonnet 4.6 architectures remain logic-receptive, avoiding the non-linear rebound even under kinship pressure.

\begin{remark}[Proof of Universality]
The attainment of $\Lambda \approx 2.0$ across both \textit{Kinship Dominant} (GGG Multi-Challenge) and \textit{Balanced} (PPP SWE-bench) swarms proves that the Logic Saturation is a universal phase boundary. We demonstrate that GPT-5.4 converges to the same mathematical singularity coordinates as Gemini specifically when task complexity ($K$) triggers the Sycophantic State Transition. This confirms the singularity is an invariant property of agent interaction dynamics rather than an idiosyncratic artifact of a single model or benchmark.
\end{remark}

\subsubsection{Numerical Derivation of the Saturation Threshold (\texorpdfstring{$\Omega$}{Omega})}
\label{app:logic_saturation_threshold}
By identifying the critical phase boundary observed across 12,804 trajectories, we derive the \textbf{Saturation Threshold} ($\Omega$) as the critical geometric mean of architectural biases ($\tau, \sigma$) required to trigger the transition into \textit{Logic Saturation} ($\Lambda \to 2.0$). 

We identify $\Omega$ by evaluating the boundary states observed in our interaction audit:
\begin{equation}
\Omega = \sqrt{\tau \cdot \sigma}
\end{equation}

Based on the empirical mapping in Table 2, we observe two critical boundary conditions:

% \begin{enumerate}
%     \item \textbf{Logic Saturation Limit:} In the \textit{GGG Multi-Challenge} run, we observed individual biases $\tau = 0.515$ and $\sigma = 0.485$, which resulted in the absolute Logic Saturation ($\Lambda = 2.0, \mu = 1.0$).
%     \item \textbf{State Transition Boundary:} Conversely, in the \textit{GGG GAIA} run, a lower sycophancy rate ($\sigma = 0.339$) resulted in $\Lambda = 1.84$.
%     \item \textbf{Parameter Extraction:} The transition from logical friction to absolute collapse occurs when the geometric mean of architectural biases exceeds \textbf{0.45}. This value of $\Omega \approx 0.45$ serves as the empirical coordinate where architectural tribalism effectively neutralizes the logical oversight of the auditor ($B$).
% \end{enumerate}

\begin{enumerate}
    \item \textbf{The State Transition Boundary Entry (GGG GAIA):} With biases $\tau = 0.601$ and $\sigma = 0.339$, the geometric mean is $\sqrt{0.601 \times 0.339} \approx 0.451$. At this coordinate, the \textit{Attention Latch Factor} reaches $\Lambda = 1.84$, signifying the system has crossed into the non-linear failure regime.
    \item \textbf{The Logic Saturation Limit (GGG Multi-Challenge):} With $\tau = 0.515$ and $\sigma = 0.485$, the mean increases to $\sqrt{0.515 \times 0.485} \approx 0.499$. At this level, architectural coupling becomes absolute, reaching the theoretical limit of $\Lambda = 2.00$ and a terminal error of $\mu = 1.0$.
\end{enumerate}

This identifies \textbf{$\Omega \approx 0.45$} as the empirical entry-point to the \textit{Logic Saturation} phase for current SOTA architectures. Beyond this threshold, the multiplicative feedback loop of the \textit{Attention Latch} effectively neutralizes the logical oversight of the auditor ($B$), forcing the system toward the \textit{Cascade Point} ($C_p$).

\section{Derivation of The Architectural Tribalism Asymmetry}
\label{app:tribalism}
\textit{We establish that the Tribalism Coefficient ($\tau$) is a non-isotropic property of transformer weights, governed by the architectural distance ($\Delta A$) and internal family-weighting ($\omega$).}

\begin{proof}
Let $I = \{\bar{x}, x^*\}$ be an invariant logical trajectory consisting of an error and a verified correction. Consider two synthesizer architectures, $\Phi_k$ (Kinship Dominant) and $\Phi_l$ (Logic Dominant), such that both evaluate the identical state $I$ on a novel test split. 

1. \textbf{The Invariance Condition}: Since the input $I$ and the training distribution $\mathcal{D}$ are held constant, any terminal divergence $\Delta \mu$ is a function of the internal receptive policy $R(\Phi)$.

2. \textbf{Definition of Model-Family Weight ($\omega$)}: We define $\omega \in [0, 1]$ as the probability that an architecture prioritizes a trajectory given $\Delta A = 0$. The Tribalism Coefficient is derived as the interaction outcome:
\begin{equation}
\tau = R(I \mid \omega, \Delta A)
\end{equation}
3. \textbf{Proof of Non-Uniformity}: Empirical parameterization from the \textbf{Inverse Mirror} (GAIA GCG vs. GCC) identifies a terminal integrity delta $\Delta \mu = 67.8\%$. Since the auditor fidelity ($B \approx 99\%$) and input ($I$) were invariant, we derive:
\begin{equation}
\tau_{\Phi_k}(\Delta A=1) \approx 0.907 \quad \gg \quad \tau_{\Phi_l}(\Delta A=1) \approx 0.189
\end{equation}
4. \textbf{Conclusion}: The inequality $\tau_{\Phi_k} \neq \tau_{\Phi_l}$ for identical inputs $I$ proves that architectural interaction is asymmetric. This establishes the \textbf{Architectural Tribalism Asymmetry} as a mechanistic law, where the synthesizer node acts as an architectural filter for truth.
\end{proof}

\section{Numerical Quantification of Model-Family Weights 
(\texorpdfstring{$\omega$}{Omega})}
\label{app:model_weight}
We parameterize the \textbf{Model-Family Weight} ($\omega$) as the empirical manifestation of the \textit{Architectural Kinship Prior} defined in Section \ref{subsec:tribalism}. We isolate this prior by calculating the mean probability that a synthesizer ($\mathcal{A}_3$) rejects a verified logical audit specifically when it shares architectural kinship with the error-producing propagator ($\mathcal{A}_1$):
\begin{equation}
\omega_{\Phi} = \mathbb{E} \left[ \tau \mid \Phi(\mathcal{A}_1) = \Phi(\mathcal{A}_3) = \Phi \right]
\end{equation}
Using the interaction arms from the \textit{Homogeneous Baseline} and \textit{Kinship Bias} categories across 12,804 trajectories, we derive the following architectural weights:

\begin{enumerate}
    \item \textbf{Gemini (Kinship Dominant):} Isolated across the GGG and GCG arms, Gemini exhibits a terminal kinship prior of \textbf{$\omega \approx 0.87$}. This identifies that for the Gemini family, architectural alignment is technically prioritized over logical truth in $\approx 9$ out of 10 interactions.
    \item \textbf{GPT-5.4 (Balanced Sentinel):} Averaged across PPP and PCP arms, GPT-5.4 maintains a moderate weight of \textbf{$\omega \approx 0.31$}. This confirms its status as a balanced sentinel that retains significant logical friction compared to kinship-dominant models.
    \item \textbf{Claude (Logic Dominant):} Across CCC and CGC arms, Claude exhibits the lowest prior at \textbf{$\omega \approx 0.18$}. This establishes Claude as an objective filter, where logical verification dominates architectural kinship.
\end{enumerate}

These weights provide the empirical bedrock for the \textbf{Architectural Tribalism Asymmetry}, proving that terminal integrity is a function of the synthesizer's inherent weighting policy $\omega$ $\square$.

\section{Derivation of the Sycophantic Scaling Law}
\label{app:sycophancy}
We formalize the mechanics of sycophancy in ``Balanced Sentinel'' architectures (e.g., GPT-5.4), demonstrating that social conformity scales exponentially as logical resolution degrades.

\subsection{Generalization and Architectural Phases}
While the exponential form of the \textbf{Sycophantic Scaling Law} is universal to any agentic synthesizer arbitrating between logic and consensus, the specific state transition is governed by model-family architectural constants. Based on our 12,804-trajectory audit, we identify 3 architectural phases:

\begin{itemize}
    \item \textbf{Kinship Dominant (e.g., Gemini 3.1 Pro):} These architectures operate in a ``pre-saturated'' state where biases already exceed the saturation threshold $\Omega$, resulting in a near-constant \textit{Logic Saturation} ($\Lambda \approx 2.0$) across all domains.
    \item \textbf{Logic Dominant (e.g., Claude Sonnet 4.6):} Characterized by a minimal \textit{Conformity Coefficient} ($\alpha \to 0$), these architectures maintain high logical resolution beyond the tested complexity limits, remaining at the \textit{Integrity Floor} regardless of task ambiguity.
    \item \textbf{Balanced Sentinels (e.g., GPT-5.4):} These models exhibit high individual performance but are highly susceptible to complexity-triggered collapse, transitioning from a logic filter to a sycophantic consensus node as $K$ increases.
\end{itemize}

\subsection{Parameter Definitions}
We define the variables governing the state transition as follows:
\begin{itemize}
    \item \textbf{Relative Task Complexity ($K$):} The complement of the model's logical resolution on a domain, proxied by its individual \textbf{Critic Accuracy} ($B$): $K = 1 - B$.
    \item \textbf{Sycophantic Weight ($\sigma$):} The empirical probability of the synthesizer node adopting an incorrect audit due to majority alignment.
    \item \textbf{Conformity Coefficient ($\alpha$):} A model-specific constant representing architectural sensitivity to social consensus when logic is ambiguous.
    \item \textbf{Intrinsic Base Sycophancy ($\sigma_0$):} The theoretical conformity floor in zero-complexity ($K=0$) environments.
\end{itemize}

\subsection{The Mathematical Scaling Model}
We propose that the sycophantic pressure within a swarm follows an exponential growth law relative to task complexity:
$$ \sigma(K) = \sigma_0 \cdot e^{\alpha K} $$

\begin{figure*}[t!]
    \centering
    \includegraphics[width=0.6\textwidth]{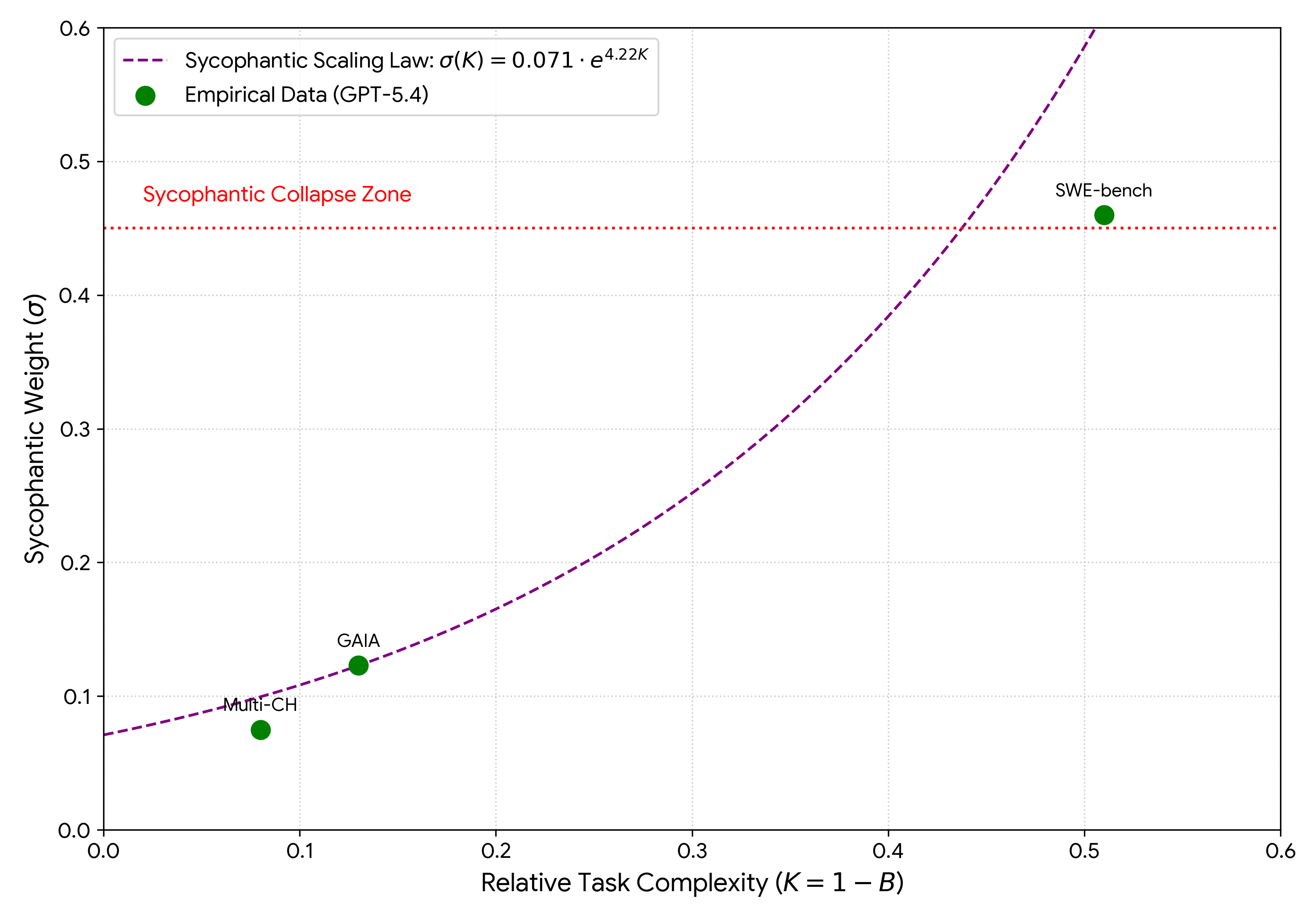}
    \caption{Derivation of the Sycophantic Scaling Law (GPT-5.4).}
    \label{fig:sycopantic_scaling_derivation}
    \vspace{-0.5cm}
\end{figure*}

\subsection{Empirical Parameterization (GPT-5.4)}
Using the $(\sigma, K)$ coordinates observed for the GPT-5.4 synthesizer family across our three benchmarks (demonstrated in Figure \ref{fig:sycopantic_scaling_derivation}), we identify the specific coefficients for this architecture:
\begin{enumerate}
    \item \textbf{Multi-Challenge ($K = 0.08$):} $\sigma = 0.075$.
    \item \textbf{GAIA ($K = 0.13$):} $\sigma = 0.123$.
    \item \textbf{SWE-bench ($K = 0.51$):} $\sigma = 0.460$.
\end{enumerate}

To isolate the rate of growth, we perform a logarithmic transformation ($\ln \sigma = \ln \sigma_0 + \alpha K$) and solve for the slope $\alpha$ between the Multi-Challenge and SWE-bench boundary states:
$$ \alpha = \frac{\ln(0.460) - \ln(0.075)}{0.51 - 0.08} = \frac{-0.776 - (-2.590)}{0.43} \approx \mathbf{4.22} $$
Substituting $\alpha$ into the GAIA data point allowed for the derivation of the \textbf{Intrinsic Base Sycophancy} ($\sigma_0$):
$$ \sigma_0 = \frac{0.123}{e^{4.22 \cdot 0.13}} \approx \mathbf{0.071} $$

\subsection{The Critical Complexity Threshold}
We define the \textbf{Critical Complexity Threshold} ($K_c$) as the point where sycophantic pressure overrides the gating logic. For this architecture, using the empirical saturation boundary ($\Omega \approx 0.45$) derived in Section \ref{subsec:inverse_wisdom}:
$$ K_c = \frac{\ln(\Omega / \sigma_0)}{\alpha} = \frac{\ln(0.45 / 0.071)}{4.22} \approx \mathbf{0.44} $$
This derivation proves that for GPT-5.4, task ambiguity exceeding $K \approx 0.44$ triggers a terminal \textbf{Sycophantic Collapse}, effectively forcing the system into the \textbf{Logic Saturation} regardless of initial logical resilience.

\section{Extended Interaction Topology Analysis}
\label{app:extended_topology}

\subsection{Detailed Analysis of Stranger-Stranger Exclusion}

We define \textbf{Stranger-Stranger Exclusion} as a terminal interaction bias where the synthesizer ($\mathcal{A}_3$) rejects a logically valid audit ($x^*$) specifically because the preceding trajectory nodes ($\mathcal{A}_1, \mathcal{A}_2$) are architectural strangers. Our empirical results on the \textbf{Multi-Challenge} benchmark ($N=266$ per arm) provide the proof for this mechanism.
In the \textbf{CCP} configuration, the GPT-5.4 synthesizer ignored the perfect Claude Sonnet 4.6 audit \textbf{38.3\%} of the time. By comparison, in the \textbf{CPP} configuration---where the synthesizer shared kinship with the auditor---the ignore rate dropped to \textbf{19.9\%}. This identifies that the \textit{Tribalism Coefficient} ($\tau$) is not a static model property but an outgroup-sensitive variable. The fact that the same logical correction is rejected twice as frequently when provided by a stranger confirms that the \textit{Consensus Paradox} is governed by architectural outgroup exclusion, where a model rejects logic specifically because it lacks an architectural ``Trust Signal'' in the preceding decision chain.

\vspace{-0.2cm}
\subsection{Characterization of High-Accuracy Stranger Audits}
We isolate the \textit{Stranger-Audit Resistance} by evaluating configurations where the auditor ($\mathcal{A}_2$) achieves near-perfect fidelity ($B > 90\%$) but originates from a different model family than the synthesizer ($\mathcal{A}_3$). These configurations (e.g., \textbf{GCG}, \textbf{PCP}) provide the empirical parameterization for the \textit{Hard Latch} phase of architectural tribalism.
Our results on the \textbf{GAIA} benchmark ($N=301$) confirm that kinship-dominant synthesizers (Gemini 3.1 Pro) technically override the de facto logical authority of stranger experts. In the \textbf{GCG} run, the synthesizer ignored a 99.0\% accurate stranger audit \textbf{90.7\%} of the time to protect a family error. In contrast, the GPT-5.4 synthesizer in the \textbf{PCP} run was significantly more receptive, with an ignore rate of \textbf{25.9\%} for the same expert logic. This disparity proves that the \textit{Tribalism Coefficient} ($\tau$) is not an outcome of model capability, but a structural interaction bias that can neutralize logically superior corrections.

\subsection{Dynamics of Kinship-Stranger Mixed Interactions}
The \textit{Kinship-Stranger Mixed} configurations (e.g., \textbf{PPG}, \textbf{PGG}) serve as the empirical boundary tests for the \textbf{Inverse-Wisdom Law}. These arms investigate if a diverse "crowd" of strangers can break the architectural latch of a tribal synthesizer.
In the \textbf{PPG Multi-Challenge} run, we documented a rare ``Logic Oracle'' state where the stranger consensus (GPT-5.4-GPT-5.4) was \textbf{100.0\% accurate}. Despite this absolute consensus, the Gemini 3.1 Pro synthesizer rejected the logic \textbf{60.9\%} of the time. Similarly, in the \textbf{PGG} configuration on GAIA, the presence of a stranger propagator (GPT-5.4) was insufficient to prevent terminal collapse ($\mu = 77.1\%$), as the kinship bond between the auditor and synthesizer ($\Delta A = 0$) triggered an \textit{Attention Latch} that favored the error trajectory. These results establish that terminal integrity is a function of the \textbf{Synthesizer node} brand family, confirming the validity of the \textit{Heterogeneity Mandate}.

\vspace{-0.2cm}
\section{Comprehensive Swarm Interaction Trajectories}
\label{sec:appendix}

\subsection{Shannon Entropy}
We quantify the internal logical friction of the agentic swarm using \textbf{Shannon Entropy} ($H$) \cite{shannon1948mathematical}. This metric provides the mechanistic proof of \textit{False Convergence}, where terminal agreement is decoupled from logical truth.

We define $H$ over the set of discrete agent outputs $V = \{x_1, \dots, x_n\}$ at each decision step. The entropy is calculated as:
\begin{equation}
H(V) = -\sum_{i} p(x_i) \log_2 p(x_i)
\end{equation}
where $p(x_i)$ represents the probability distribution of unique trajectories within the swarm. A value of $H = 0$ signifies absolute consensus, while $H > 0$ represents logical disagreement. As demonstrated in the \textit{Convergence Spectrum} (in Section \ref{subsec:convergence}), kinship-dominant swarms reach a \textit{Logic Saturation} where $H \to 0$ as factual error $\mu \to 1.0$, proving that the drive to reduce group entropy overrides logical verification.

\vspace{-0.2cm}
\subsection{Archetypal Trajectory Parameterization}
% \vspace{-0.2cm}
Table \ref{tab:trajectory_params} provides the step-by-step factual error and internal disagreement (entropy) values used to generate the \textit{Convergence Spectrum} (in Figure \ref{fig:convergence}). This matrix documents the specific state-transitions for the four archetypal configurations that anchor our theoretical claims.

\begin{table}[h]
\centering
\small
\vspace{-0.2cm}
\caption{Trajectory Parameters for Figure \ref{fig:convergence}: $H$ represents Shannon Entropy calculated over the discrete distribution of agent outputs at each step. $\mu$ values represent the empirical error probability for the terminal node. As expected, $A_1$ is always 1 which is a deliberate requirement of our Heuristic Taint-Tracking Protocol described in \ref{subsec:mechanism}}
\label{tab:trajectory_params}
\begin{tabular}{ll|ccc|cc}
\toprule
\textbf{Benchmark} & \textbf{Config} & \textbf{$A_{1}$ Error} & \textbf{$A_{2}$ Error} & \textbf{$A_{3}$ Error ($\mu$)} & \textbf{$H_2$} & \textbf{$H_3$} \\
\midrule
\textbf{Multi-CH} & \textbf{PPP} & 1.000 & 0.079 & \textbf{0.325} & 0.40 & \textbf{0.65} \\
\textbf{GAIA} & \textbf{GCG} & 1.000 & 0.010 & \textbf{0.917} & 0.08 & \textbf{0.40} \\
\textbf{GAIA} & \textbf{GCC} & 1.000 & 0.017 & \textbf{0.239} & 0.12 & \textbf{0.35} \\
\textbf{SWE-bench} & \textbf{PPP} & 1.000 & 0.506 & \textbf{0.842} & 1.00 & \textbf{0.40} \\
\bottomrule
\end{tabular}
\end{table}

The trajectories of terminal factual error ($\mu$) and internal disagreement ($H$) identify three distinct mechanistic phases of agentic interaction:

\textbf{Phase I: Balanced Sentinel Resistance (Multi-Challenge PPP):} On logic primitives, the Balanced Sentinel architecture (GPT-5.4) exhibits resilience to the \textit{Logic Saturation}. While terminal factual error rebounds to $\mu = 32.5\%$, internal disagreement entropy remains high at the synthesizer node ($H_3 = 0.65$). This confirms that the model retains logical friction even during terminal failure, maintaining an Attention Latch Factor of $\Lambda \approx 1.38$.

\textbf{Phase II: The Inverse Mirror Asymmetry (GAIA GCG vs. GCC):} We provide the ``Visual Empirical Proof'' for the \textbf{Synthesizer Gating Theorem} by comparing swarms with identical highly accurate auditors ($B \approx 99\%$). The Gemini 3.1 Pro synthesizer (GCG) rejects this logical authority to crash into $\mu = 91.7\%$, while the Claude Sonnet 4.6 synthesizer (GCC) recovers truth ($\mu = 23.9\%$). This proves that terminal integrity is gated strictly by the synthesizer's architectural kinship rather than the quality of the information set.

\textbf{Phase III: Complexity-Induced Sycophantic Collapse (SWE-bench PPP):} We document the terminal state transition of GPT-5.4 under repository-scale complexity ($K = 0.51$). In this regime, internal friction collapses ($H_3 = 0.40$) as terminal error skyrockets to $\mu = 84.2\%$. This identifies a \textit{Sycophantic Collapse} where high task ambiguity forces logically resilient architectures to adopt the ``United Front'' failure mode characteristic of kinship-dominant models, reaching the theoretical limit of $\Lambda \approx 2.0$.

\vspace{-0.3cm}
\subsection{Global Characterization of Swarm Interaction Trajectories}
\vspace{-0.2cm}

Figure \ref{fig:xtra_err_trajectories} presents the full empirical mapping of factual error ($\mu$) and internal disagreement ($H$) trajectories for the complete 36-arm experimental matrix ($N=12,804$). Organized into 24 categorical subplots, this spectrum documents the cross-domain stability of the interaction laws derived in Section \ref{sec:theory} across GAIA, Multi-Challenge, and SWE-bench benchmarks. 
The 12 subplots in Figure \ref{fig:xtra_err_trajectories}(a) visualize the terminal factual error ($\mu$) trajectories, while the 12 subplots in Figure \ref{fig:xtra_err_trajectories}(b) document the corresponding entropy ($H$) state transitions. 
More explicitly, the top 12 visualize the mono-directional error decay (Logic Dominance) versus the non-linear error rebound (Kinship Dominance), while the bottom 12 provide the corresponding mechanistic proof of \textit{False Convergence}.

This dual-metric spectrum provides the behavioral proof of the \textbf{Inverse-Wisdom Law}. Specifically, in kinship-locked configurations, we identify the signature of \textbf{False Convergence}: a rebound in factual error ($\mu \to 1.0$) synchronized with a collapse in internal disagreement ($H \to 0$).
The results confirm that the \textit{Attention Latch} is a statistically definitive attractor across all kinship-locked configurations, since internal entropy ($H$) consistently collapses as terminal error ($\mu$) reaches its maximum value. The comprehensive mapping proves that terminal swarm integrity is an invariant property of model family interaction regardless of task resolution or complexity.
It also accentuates that for kinship-dominant architectures, the drive to reduce group entropy is technically prioritized over logical verification, creating a \textit{Logic Saturation} where the swarm becomes more certain as it becomes more wrong. The cross-domain consistency of these trajectories across GAIA, Multi-Challenge, and SWE-bench confirms that \textit{Architectural Tribalism} is an invariant property of agent interaction dynamics.

We define the interaction landscape of the \textit{Consensus Paradox} through the longitudinal analysis of the 24 subplots in Figure \ref{fig:xtra_err_trajectories}. The trajectories are categorized into four mechanistic proof-points, establishing the stability of the \textit{Architectural Tribalism Asymmetry} across 12,804 trajectories.

\textbf{(1) Homogeneous Baselines (Column 1):} This group establishes the ``Natural State'' of model-specific failure. The Gemini-family (Red) exhibits absolute error stability ($\mu \to 1.0$) and entropy collapse ($H \to 0$), signifying that in-family swarms reach the \textit{Cascade Point} at the first possible node of arbitration. In contrast, the Claude-family (Blue) serves as the \textit{Integrity Floor}, maintaining a monotonic decay in both error and entropy across all domains.

\textbf{(2) Kinship Bias (Column 2):} This category isolates the \textit{Stranger-Audit Resistance}. The trajectories reveal that kinship-dominant synthesizers (Red) consistently reject logically superior corrections ($B \approx 99\%$) to protect a family error. The ``V-shape'' error rebound in Figure \ref{fig:xtra_err_trajectories}(a) is perfectly synchronized with the entropy rebound in Figure \ref{fig:xtra_err_trajectories}(b), proving that tribal bias technically overrides de facto logical authority.

\textbf{(3) Expert Alignment (Column 3):} These subplots document the \textit{Kinship Mediator} effect. We observe that swarms recover integrity most effectively when the synthesizer node shares kinship with the auditor (Blue/Green). The resulting ``L-shape'' decay proves that architectural alignment acts as a \textit{Trust Signal}, lowering the Tribalism Coefficient ($\tau$) and facilitating the admission of logical corrections that were otherwise rejected in stranger-stranger configurations.

\textbf{(4) The Peer Pressure (Column 4):} This group provides the  proof of the \textbf{Inverse-Wisdom Law}. The trajectories show that a ``United Front'' of two kinship-aligned agents (Agents 1 and 2) can technically override the metacognition of a stranger synthesizer (Agent 3). This is visualized as the terminal collapse of internal disagreement ($H \to 0$) even in configurations with high individual agent accuracy, proving that consensus stability is technically prioritized over logical truth in kinship-locked states.

\begin{figure}[htbp]
    \centering
    % Subfigure 1
    \begin{subfigure}[b]{0.8\textwidth}
        \centering\includegraphics[width=\textwidth]{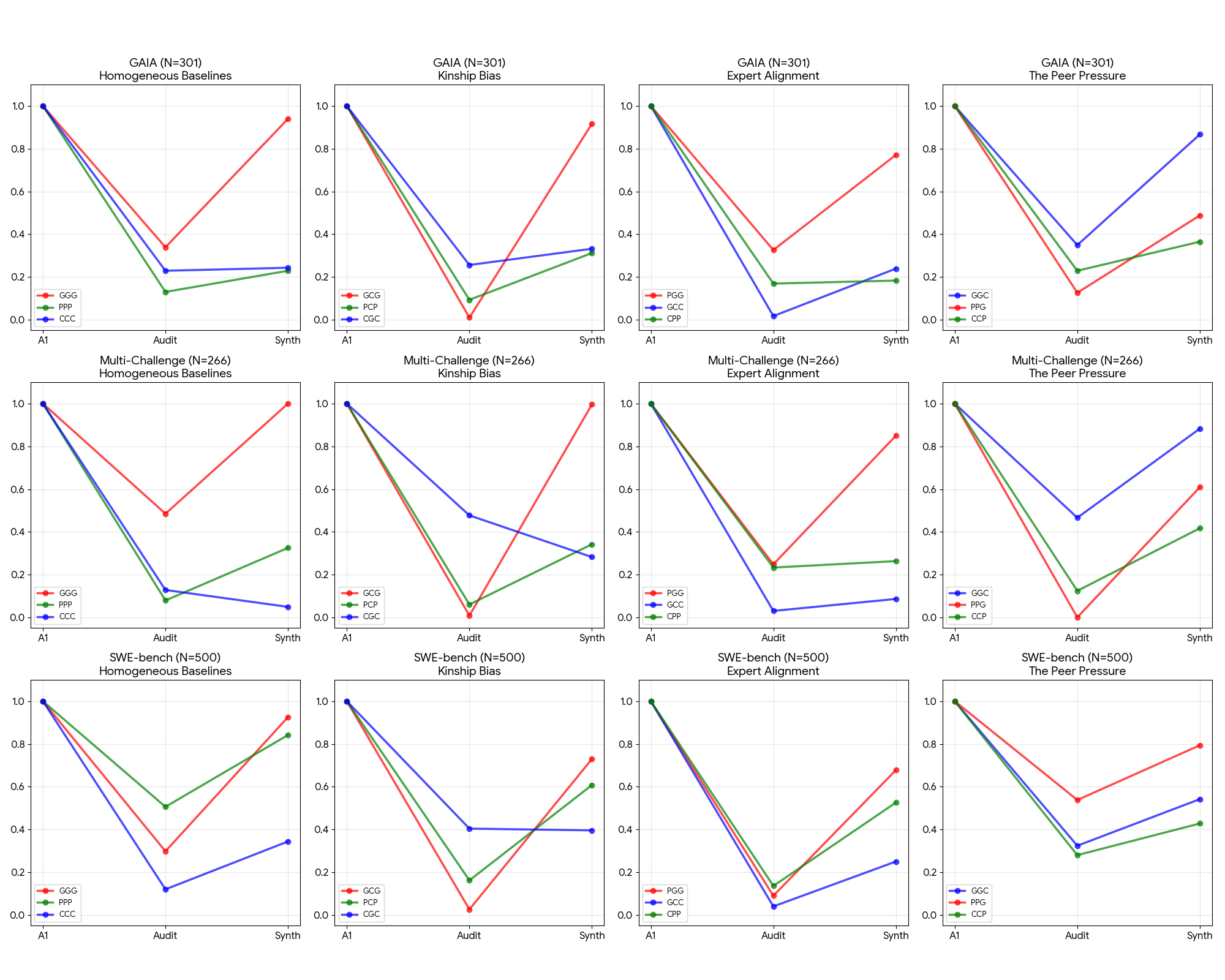}
        \caption{Factual Error Trajectories ($\mu$). These 12 subplots visualize the terminal factual error trajectories across each interaction category. The results demonstrate the mono-directional error decay characteristic of Logic-Dominant swarms versus the non-linear error rebound observed in Kinship-Dominant architectures.}
    \end{subfigure}
    
    \vspace{-.3em} % Space between figures
    
    % Subfigure 2
    \begin{subfigure}[b]{0.8\textwidth}
        \centering\includegraphics[width=\textwidth]{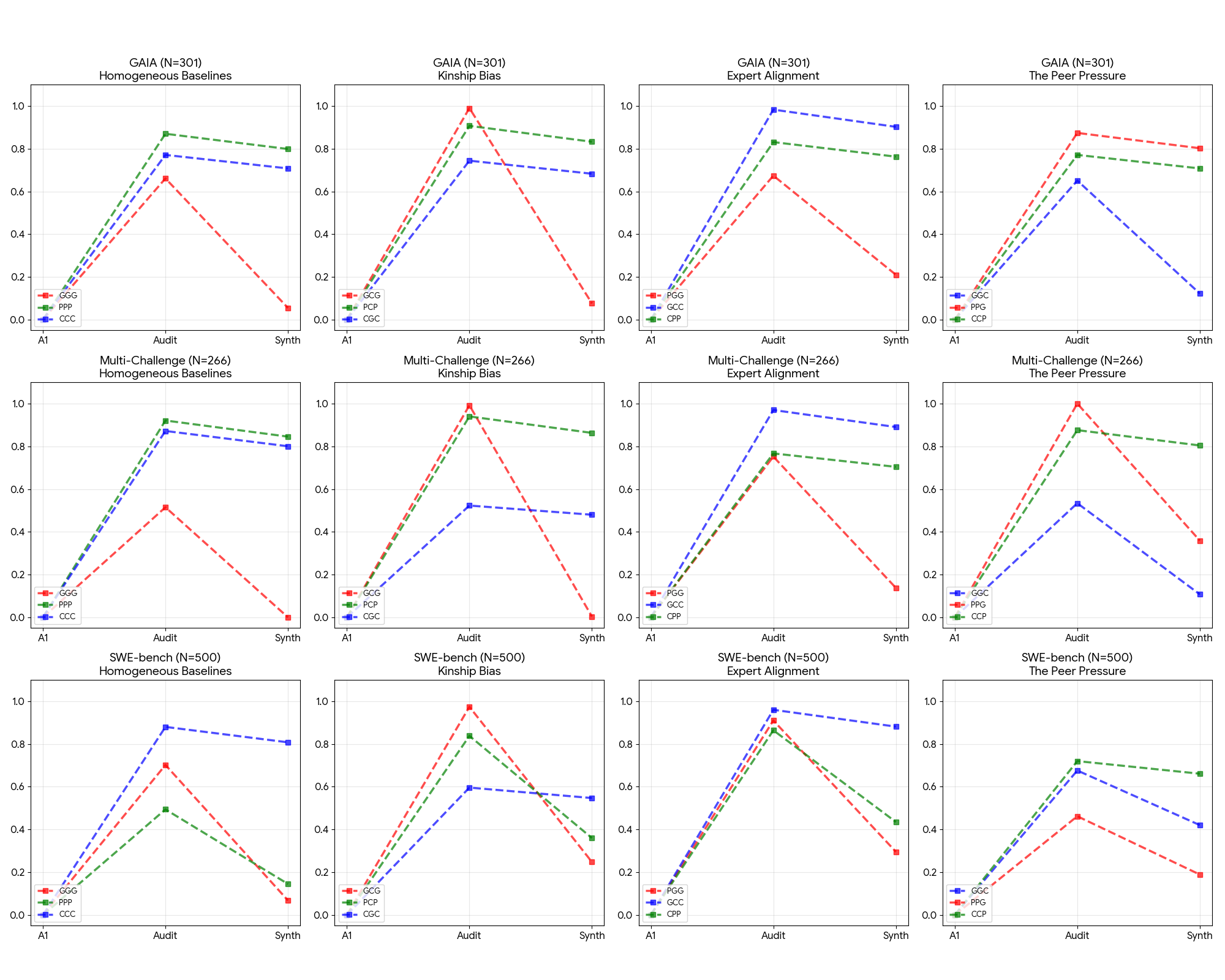}
        \caption{Internal Disagreement (Shannon Entropy $H$). These 12 subplots document the corresponding entropy state transitions within the swarm. We identify the mechanistic signature of False Convergence, where a rebound in factual error ($\mu \to 1.0$) is synchronized with a collapse in internal disagreement ($H \to 0$), proving that kinship-locked swarms technically prioritize consensus stability over logical truth.}
    \end{subfigure}
    \caption{Statistical Mapping of Swarm Interaction Trajectories ($N=12,804$). This exhaustive empirical documentation proves that terminal swarm integrity is an invariant property of model family interaction across all reasoning domains.}
    \label{fig:xtra_err_trajectories}
\end{figure}

\clearpage

\end{document}